\documentclass[12pt,draftcls,peerreviewca,onecolumn]{IEEEtran}

\usepackage{cite}
\usepackage{amsmath,amssymb,amsfonts}
\usepackage{algorithmic}
\usepackage{graphicx}
\usepackage{textcomp}
\usepackage{algorithm}
\usepackage{amsfonts}
\usepackage{cite}
\usepackage{graphicx}
\usepackage{amssymb}
\usepackage{caption}
\usepackage{epsfig}
\usepackage{subfigure}
\usepackage{amsmath}
\usepackage{array}
\usepackage{color}
\usepackage{amsmath}
\usepackage{longtable}
\usepackage{amsthm}
\usepackage{algorithmic}
\usepackage{algorithm}
\usepackage{cases}
\usepackage{mathrsfs}
\usepackage{psfrag}
\usepackage{url}
\usepackage{multicol}
\usepackage{setspace}
\usepackage{float}

\usepackage[dvipsnames, svgnames, x11names]{xcolor} 
\usepackage{ulem}
\usepackage{footnote}
\usepackage{etoolbox}
\makeatletter
\patchcmd{\@makecaption}
  {\scshape}
  {}
  {}
  {}
\makeatletter
\patchcmd{\@makecaption}
  {\\}
  {.\ }
  {}
  {}
\makeatother

\IEEEoverridecommandlockouts

\def\BibTeX{{\rm B\kern-.05em{\sc i\kern-.025em b}\kern-.08em
    T\kern-.1667em\lower.7ex\hbox{E}\kern-.125emX}}

\newtheorem{Thm}{Theorem}
\newtheorem{Lem}{Lemma}
\newtheorem{Asump}{Assumption}
\newtheorem{Prob}{Problem}

\newtheorem{Rmk}{Remark}

\newcommand{\setNbar}{\bar{\mathcal N}}

\allowdisplaybreaks[4]  

\UseRawInputEncoding
\begin{document}

\title{An Optimization Framework for Federated  Edge Learning}

\author{\IEEEauthorblockN{Yangchen Li, Ying Cui, and Vincent Lau}\thanks{Yangchen Li and Ying Cui are with Shanghai Jiao Tong University, China.
Vincent Lau is with HKUST, Hong Kong.
This paper will be presented in part at IEEE GLOBECOM 2021 \cite{GenQSGD-GC}.}}
\maketitle

\begin{abstract}
The optimal design of federated learning (FL) algorithms for solving general machine learning (ML) problems in practical edge computing systems with quantized message passing remains an open problem.
This paper considers an edge computing system where the server and workers have possibly different computing and communication capabilities and employ quantization before transmitting messages.
To explore the full potential of FL in such an edge computing system, we first present a general FL algorithm, namely GenQSGD, parameterized by the numbers of global and local iterations, mini-batch size, and step size sequence.
Then, we analyze its convergence for an arbitrary step size sequence and specify the convergence results under three commonly adopted step size rules, namely the constant, exponential, and diminishing step size rules.
Next, we optimize the algorithm parameters to minimize the energy cost under the time constraint and convergence error constraint, with the focus on the overall implementing process of FL.
Specifically, for any given step size sequence under each considered step size rule, we optimize the numbers of global and local iterations and mini-batch size to optimally implement FL for applications with preset step size sequences.
We also optimize the step size sequence along with these algorithm parameters to explore the full potential of FL.
The resulting optimization problems are challenging non-convex problems with non-differentiable constraint functions.
We propose iterative algorithms to obtain KKT points using  general inner approximation (GIA) and tricks for solving complementary geometric programming (CGP).
Finally, we numerically demonstrate the remarkable gains of GenQSGD with optimized algorithm parameters over existing FL algorithms and reveal the significance of optimally designing general FL algorithms.
\end{abstract}


\setcounter{page}{1}
\newpage
\section{Introduction}
With the development of mobile Internet and the Internet of Things (IoT), a massive amount of data is generated at the edge of wireless networks and stored in a distributed manner.
Leveraging on the emerging machine learning (ML) technologies, these distributed databases, usually containing privacy-sensitive data, can be used to train models for intelligent applications,
such as keyboard search suggestions, human activity recognition, human mobility prediction, ranking browser history suggestions, and patient clustering.
Thus, it may be impossible or undesirable to upload distributed databases to a central server due to energy and bandwidth limitations or privacy concerns.
Recent years have witnessed the growing interest in federated learning (FL) in edge computing systems, also referred to as federated edge learning,
where the server periodically updates the global model by aggregating and averaging the local models trained and uploaded by the workers \cite{EdgeFL,EdgeIntelligence}.
Therefore, FL can successfully protect data privacy for privacy-sensitive applications and improve communication efficiency.

Typical FL algorithms, such as Parallel Mini-batch SGD (PM-SGD) \cite{PMSGD}, Federated Averaging (FedAvg) \cite{FedAvg}, and Parallel Restarted SGD (PR-SGD) \cite{YuHao}, are designed based on stochastic gradient descent (SGD) algorithms\footnote{Some recent proposed FL algorithms are based on stochastic successive convex approximation (SSCA) \cite{YCC-SSCA}, which are out of the scope of this paper.}
and are usually parameterized by the numbers of global and local iterations, mini-batch size, and step size sequence.
Specifically, in PM-SGD, each worker utilizes one mini-batch in each local iteration and conducts only one local iteration within a global iteration;
in FedAvg, participated workers utilize all their local samples and conduct the same number of local iterations within a global iteration;
and in PR-SGD, each worker
conducts multiple local iterations within a global iteration.
In addition, \cite{PMSGD} and \cite{FedAvg} adopt a constant step size when implementing global iterations,
and \cite{PMSGD} analyzes the convergence of PM-SGD for convex ML problems;
\cite{YuHao} adopts both constant and time-varying step size sequences when implementing global iterations and analyzes the convergence of PR-SGD for non-convex ML problems.

The main limitations of the above FL algorithms lie in the following two aspects.
Firstly, \cite{PMSGD,FedAvg,YuHao} assume that the server and workers send accurate model updates, which may contain a large number of information bits.
Due to limited communication resources, accurate model updates may not be delivered timely.
To address this issue, \cite{FedPAQ} and \cite{UVeQFed} propose to quantize the local model updates before transmitting them to the server.
In particular, in \cite{FedPAQ}, the authors propose an FL algorithm, namely FedPAQ, where the participated workers employ general random quantization on the local model updates;
in \cite{UVeQFed}, the authors propose a quantization scheme, namely UVeQFed, for the local model updates generated by FedAvg.
Note that \cite{FedPAQ,UVeQFed} assume that the server sends the exact model updates,
which may consume a lot of transmission time and energy.
Secondly, in \cite{PMSGD,FedAvg,YuHao,FedPAQ,UVeQFed}, some algorithm parameters (such as the step size sequences in \cite{FedAvg,YuHao,FedPAQ,UVeQFed}) are usually chosen empirically and experimentally, and some (such as the number of local iterations for all workers in \cite{PMSGD}) are fixed.
Choosing parameters via numerous experiments may be time-consuming in practice and does not have any theoretical guarantee.
To deal with this issue, the authors in \cite{AlgDes1-1,ResAlls2-1,ResAlls3-2,ResAlls4-1} consider the optimization of the algorithm parameters to minimize the convergence error\cite{AlgDes1-1}, energy consumption \cite{ResAlls2-1}, completion time \cite{ResAlls3-2}, and {the} weighted sum of the energy consumption and completion time \cite{ResAlls4-1}, respectively.
Unfortunately, the results in \cite{AlgDes1-1,ResAlls2-1,ResAlls3-2,ResAlls4-1} hold only for strongly-convex ML problems and accurate model updates, as the formulations rely on the convergence errors derived under these restrictions.
However, in many ML applications, the training problems are non-convex.

In summary, there are two issues unresolved in designing FL algorithms for general (not necessarily convex) ML problems in the presence of quantization errors for global and local model updates.
First, the convergence analysis of an FL algorithm with arbitrary step size rules and algorithm parameters is unknown.
Second, the optimal choice of FL algorithm parameters for reducing the convergence error as well as computing and communication costs when implementing FL in practical edge computing systems is open.
This paper aims to shed some light on the two fundamental problems.
Specifically, we consider a practical edge computing system where the server and workers may have different computing and communication capabilities and employ quantization on their model updates.
We concentrate on the overall implementing process of FL and optimize the algorithm parameters of a general FL algorithm for general ML problems.
Note that short time-scale optimal computing and communication resource allocation for one global iteration \cite{ResAlls1-1,ResAlls2-2,ResAlls2-3,ResAlls3-1,MIMO} is out of the scope of this paper.
The main contributions of this paper are summarized below.

\begin{itemize}
\item[$\bullet$]\textbf{General FL Algorithm with Quantized Message Passing:} We present a general quantized parallel mini-batch SGD algorithm, namely GenQSGD,
    where the server and all workers send quantized model updates to effectively adapt to the communication capabilities of the server and all workers. GenQSGD is parameterized by the numbers of global and local iterations, mini-batch size, and step size sequence,
    which can be flexibly chosen to adequately adapt to the computing capabilities of the server and all workers and effectively improve the convergence speed.
    On the contrary, the server \cite{PMSGD,FedAvg,YuHao,FedPAQ} and workers \cite{PMSGD,FedAvg,YuHao} send accurate model updates.
    Besides, some algorithm parameters are fixed in \cite{PMSGD,FedAvg,YuHao,FedPAQ}, and the quantization parameters for all workers are identical in \cite{FedPAQ}.
\item[$\bullet$]{\textbf{General Convergence Analysis:}} We analyze the convergence of GenQSGD with an arbitrary step size sequence and specify the results under three commonly adopted step size rules, i.e., the constant step size rule \cite{SubgradientMethod,NonlinearProgramming}, exponential step size rule \cite{ExpStepSize}, and diminishing step size rule \cite{SubgradientMethod,NonlinearProgramming}.
    {In contrast, the convergences of PM-SGD and PR-SGD {are characterized only} for accurate model updates in \cite{PMSGD} and \cite{YuHao} respectively, and the convergence of FedPAQ {is analyzed solely for specific step size rules} in \cite{FedPAQ}.}
\item[$\bullet$]\textbf{Optimization of Algorithm Parameters for Fixed Step Size Sequences:} Considering applications with preset step size sequences, we optimize the numbers of global and local iterations and mini-batch size for a fixed step size sequence under each of the three step size rules to minimize the energy cost under the time constraint and convergence error constraint.
    The three optimization problems are challenging non-convex problems with non-differentiable constraint functions.
    We propose iterative algorithms to solve these problems using general inner approximation (GIA) \cite{GIA} and tricks for solving complementary geometric programming (CGP) \cite{CGP}.
    We also characterize their convergences to KKT points.
\item[$\bullet$]\textbf{Optimization of All Algorithm Parameters:} Aiming to explore the full potential of FL, we optimize the numbers of global and local iterations, mini-batch size, and step size sequence to minimize the energy cost under the time constraint, convergence error constraint, and step size constraint.
    Due to the presence of a dimension-varying vector variable, the corresponding non-convex optimization problem with non-differentiable constraint functions is even more complicated than those for fixed step size sequences.
    To tackle this extra challenge,
    we first show that the constant step size rule achieves the minimum convergence error among all step size rules.
    Then, based on this property, we equivalently transform the original problem to a more tractable one and propose an iterative algorithm to obtain a KKT point using GIA and CGP.
    Note that in \cite{PMSGD,FedAvg,YuHao}, some algorithm parameters are chosen via time-consuming experiments, which do not have any theoretical guarantees.
\item[$\bullet$]\textbf{Numerical Results:} We numerically demonstrate the convergence of GenQSGD
    and illustrate that the optimization-based GenQSGD can achieve a trade-off among the time cost, energy cost, and convergence error.
    We also numerically show remarkable gains of the proposed GenQSGD with optimized algorithm parameters over existing FL algorithms.
\end{itemize}

{\bf{Notation}}:
We use boldface letters (e.g., $\mathbf x$), non-boldface letters (e.g., $x$ or $X$), and calligraphic letters (e.g., $\mathcal X$) to represent vectors, scalar constants, and sets, respectively.
$\|\cdot\|_2$, $\mathbb E[\cdot]$, and $\mathbb I[\cdot]$ represent the $l_2$-norm, expectation, and indicator function, respectively.
The set of real numbers, positive real numbers, and positive integers are denoted by $\mathbb R$, $\mathbb R_+$, and $\mathbb Z_+$, respectively.
All-ones vector is denoted by $\mathbf 1$.

\section{System Model}
We consider an edge computing system consisting of one server and $N$ workers, which are connected via wireless links.
Let $0$ and $\mathcal N\triangleq\{1,2,\cdots,N\}$ denote the server index and the set of worker indices, respectively.
For ease of exposition, we also denote $\setNbar\triangleq\{0\}\cup\mathcal N$.
We assume that each worker $n\in\mathcal N$ holds $I_n$ samples,
denoted by $\xi_i,i\in\mathcal I_n$ with $\left|\mathcal I_n\right|=I_n$.
Note that $\xi_i,i\in\mathcal I_n$ can be viewed as the realizations of a random variable, denoted by $\zeta_n$.
The server and $N$ workers aim to collaboratively train a global model by solving an ML problem based on the local data stored on the $N$ workers.
The global model is parameterized by a $D$-dimensional vector ${\bf x}\in \mathbb R^{D}$.
Specifically, for a given $\mathbf x\in\mathbb R^D$, define the loss incurred by $\zeta_n$ as $F({\mathbf x};\zeta_n)$ and define the expected loss as $f_n({\mathbf x})\triangleq\mathbb E\left[F({\mathbf x};\zeta_n)\right]$, with the expectation taken with respect to the distribution of $\zeta_n$, for all $n\in\mathcal N$.
Then, the expected risk function $f:\mathbb R^{D}\rightarrow\mathbb R$ of the model parameters ${\mathbf x}\in\mathbb R^{D}$ is defined as:
\begin{align}\label{eq:expected_loss}
f({\mathbf x})\triangleq\frac{1}{N}\sum_{n\in\mathcal N}f_n({\mathbf x}).
\end{align}
To be general, we do not assume $f({\bf x})$ to be convex.
Our goal is to minimize the expected risk function with respect to the model parameters ${\bf x}$ in the edge computing system.
\begin{Prob}[ML Problem]\label{Prob1}
\begin{align}\label{Obj_ML}
f^*\triangleq\min_{{\mathbf x}}f({\bf x})
\end{align}
where $f({\mathbf x})$ is given by \eqref{eq:expected_loss}.\footnote{A constrained ML problem can be transformed into an unconstrained ML problem by augmenting the objective function of the original problem with a weighted sum of the constraint functions.
The weights can be numerically adjusted to guarantee that the constraints are satisfied.
Thus, the proposed GenQSGD can be applied to solve a constrained ML problem.}
\end{Prob}
Problem~\ref{Prob1} is an unconstrained problem which may be convex or non-convex.
The goal of solving an unconstrained problem is generally to design an iterative algorithm to obtain a stationary point.\footnote{Note that any stationary point of a convex problem is globally optimal, and a stationary point of a non-convex problem may be a locally optimal point, globally optimal point, or saddle point.}
The server and $N$ workers all have computing and communication capabilities.
Let $F_0$ and $F_n$ denote the CPU frequencies (cycles/s) of the server and worker $n\in\mathcal N$, respectively.
Let $p_0$ and $p_n$ denote the transmission powers of the server and worker $n\in\mathcal N$, respectively.
In the process of collaborative training,
the server multicasts messages to the $N$ workers at an average rate $r_0$ (b/s) over the whole frequency band,
and the $N$ workers transmit their messages to the server at average transmission rates $r_n,n\in\mathcal N$ (b/s) using frequency division multiple access (FDMA).
The server and all $N$ workers employ quantization before transmitting messages.
Throughout this paper, for each node, we consider an arbitrary random quantizer $\mathbf Q(\cdot;s):\mathbb R^D\rightarrow\mathbb R^D$, which has a tunable quantization parameter $s\in\mathbb Z_+$ (corresponding to the number of quantization levels)\footnote{Usually, $s$ is the number of quantization levels or its increasing function.}
and satisfies the following assumption \cite[Assumption 1]{FedPAQ}.
\begin{Asump}[Random Quantization]\label{Asump:Quantization}
For all $\mathbf y\in\mathbb R^D$ and $s\in\mathbb Z_+$, $\mathbf Q(\cdot;s)$ satisfies: (i) $\mathbb E\left[{\mathbf Q}(\mathbf y;s)\right]=\mathbf y$ and (ii) $\mathbb E\left[\left\|\mathbf Q(\mathbf y;s)-\mathbf y\right\|_2^2\right]\leq q_s\left\|\mathbf y\right\|_2^2$,
for some constant\footnote{$q_s$ depends on (usually decreases with) $s$.} $q_s>0$.
\end{Asump}
For an input vector $\mathbf y\in\mathbb R^D$, the number of bits to specify the quantized vector ${\mathbf Q}(\mathbf y;s)$, i.e., represent $\mathbf y$, is denoted by $M_s$ (bits).
In this paper, we use $s_0\in\mathbb Z_+$ and $s_n\in\mathbb Z_+$ to denote the random quantization parameters for the server and worker $n\in\mathcal N$, respectively.
\begin{Rmk}[General Edge Computing System]
The edge computing system considered here is general in the sense that the system parameters, i.e., $F_n,p_n,r_n,s_n,n\in\bar{\mathcal N}$, can be different.
\end{Rmk}

\section{Algorithm Description and Convergence Analysis for GenQSGD}\label{Sec:Alg}
In this section, we first present a general quantized parallel mini-batch SGD algorithm for solving Problem~\ref{Prob1} in the edge computing system.
Then, we analyze its convergence.

\subsection{Algorithm Description}
The proposed GenQSGD algorithm is parameterized by $\mathbf K\triangleq\left(K_n\right)_{n\in\setNbar}\in\mathbb Z_+^{N+1}$, $B\in\mathbb Z_+$,
and $\mathbf\Gamma\triangleq\left(\gamma^{(k_0)}\right)_{k_0\in\mathcal K_0}\in\mathbb R_+^{K_0}$, where $\mathcal K_0\triangleq\{1,2,\cdots,K_0\}$.
Specifically, $K_0$ represents the number of global iterations,
$K_n$ represents the number of local iterations executed by worker $n\in\mathcal N$ within one global iteration,
$B$ represents the local mini-batch size used for local iterations at each worker,
and $\gamma^{(k_0)}$ represents the (constant) step size\footnote{For ease of exposition, we consider a constant step size for the local iterations within one global iteration, as in \cite{YuHao}, \cite{FedPAQ}.
The convergence analysis and optimization results in this paper can be readily extended to the case with any step size rule.} used in the local iterations within the $k_0$-th global iteration.
Besides, let $\mathcal K_n\triangleq\{1,2,\cdots,K_n\}$ denote the local iteration index set for worker $n\in\mathcal N$.

For all $k_0\in\mathcal K_0$, $\hat{\mathbf x}^{(k_0)}\in\mathbb R^D$ denotes the global model recovered by all $N$ workers
(and the server, if it needs to obtain the final global model) at the beginning of the $k_0$-th global iteration,
${\mathbf x}_n^{(k_0,0)}\in\mathbb R^D$ denotes the initial local model of worker $n\in\mathcal N$ at the beginning of the $k_0$-th global iteration,
and $\Delta\hat{\mathbf x}^{(k_0)}\in\mathbb R^D$  denotes the average of the quantized overall local model updates (termed the global model update) at the $k_0$-th global iteration.
For all $k_0\in\mathcal K_0,k_n\in\mathcal K_n,n\in\mathcal N$, $\mathbf x_n^{(k_0,k_n)}\in\mathbb R^D$
and $\mathcal B_n^{(k_0,k_n)}\subseteq\left\{\xi_i:i\in\mathcal I_n\right\}$ denote the local model of worker $n$
and the mini-batch used by worker $n$, respectively, at the $k_n$-th local iteration within the $k_0$-th global iteration.
\begin{algorithm}[t]
\caption{GenQSGD}\vspace{-10pt}
\label{Alg:GenQSGD}
\begin{multicols}{2}
{\small
{\bf{Input:}} $\mathbf K\in\mathbb Z_+^{N+1}$, $B\in\mathbb Z_+$, and $\mathbf \Gamma\in\mathbb R_+^{K_0}$.\\
{\bf{Output:}} $\mathbf x^*\left(\mathbf K,B,\mathbf\Gamma\right)$.
}
\begin{algorithmic}[1]
{\small
\STATE {\bf{Initialize:}} The server generates $\mathbf x_0^{(0)}$, sets $\Delta\hat{\mathbf x}^{(0)}=\mathbf x_0^{(0)}$, and sends $\mathbf Q(\Delta\hat{\mathbf x}^{(0)};s_0)$ to all $N$ workers.
The $N$ workers set $\hat{\mathbf x}^{(0)}=0$.
\FOR {$k_0=1,2,\cdots,K_0$}
    \FOR {worker $n\in{\mathcal N}$}\label{Step:begin_global_iteration}
        \STATE Compute $\hat{\mathbf x}^{(k_0)}$ according to:
        \begin{align}\label{eq:RecoveredLocalModel}
        \!\!\!\!\hat{\mathbf x}^{(k_0)}\!\!:=\!\hat{\mathbf x}^{(k_0-1)}\!\!+\!\gamma^{(k_0-1)}\mathbf Q(\!\Delta\hat{\mathbf x}^{(k_0-1)}\!;\!s_0),\!\!
        \end{align}
        and set ${\mathbf x}_n^{(k_0,0)}=\hat{\mathbf x}^{(k_0)}$.\label{Step:initialize_local_model}
        \FOR {$k_n=1,2,\cdots,K_n$}\label{Step:begin_local_iteration}
            \STATE Randomly select a mini-batch $\mathcal B_n^{(k_0,k_n)}$ and update ${\bf x}_n^{(k_0,k_n)}$ according to:
            \begin{align}\label{eq:local_update}
            &{\mathbf x}_n^{(k_0,k_n)}:={\bf x}_n^{(k_0,k_n-1)}\nonumber\\
            &-\frac{\gamma^{(k_0)}}{B}\!\!\!\!\!\!\sum_{\xi\in\mathcal B_n^{(k_0,k_n)}}\!\!\!\!\!\!\nabla{F\left({\bf x}_n^{(k_0,k_n-1)};\xi\right)}.
            \end{align}
        \ENDFOR\label{Step:end_local_iteration}
        \STATE Compute $\frac{{\mathbf x}_n^{(k_0,K_n)}-\hat{\mathbf x}^{(k_0)}}{\gamma^{(k_0)}}$, and send $\mathbf Q\left(\frac{{\mathbf x}_n^{(k_0,K_n)}-\hat{\mathbf x}^{(k_0)}}{\gamma^{(k_0)}};s_n\right)$ to the server.\label{Step:upload}
    \ENDFOR\label{Step:end_global_iteration}
    \STATE The server computes $\Delta\hat{\mathbf x}^{(k_0)}$ according to:
    \begin{align}\label{eq:GlobalUpdate}
    \!\!\Delta\hat{\mathbf x}^{(k_0)}\!\!:=\!\frac{1}{N}\!\!\sum_{n\in\mathcal N}\!\!\mathbf Q\!\left(\frac{{\mathbf x}_n^{(k_0,K_n)}-\hat{\mathbf x}^{(k_0)}}{\gamma^{(k_0)}};s_n\right)\!,\!\!\!
    \end{align}
    and sends $\mathbf Q(\Delta\hat{\mathbf x}^{(k_0)};s_0)$ to all $N$ workers.\label{Step:broadcast}
\ENDFOR\label{Step:end_global_iteration}
\STATE The server\footnote{The server needs to compute $\hat{\mathbf x}^{(k_0)}$ and $\hat{\mathbf x}^{(K_0+1)}$ according to \eqref{eq:RecoveredLocalModel} in Step 10 and Step 12, respectively, if it needs to obtain the final global model.} and all $N$ workers compute $\hat{\mathbf x}^{(K_0+1)}$ according to \eqref{eq:RecoveredLocalModel}, and set $\mathbf x^*\left(\mathbf K,B,\mathbf\Gamma\right)=\hat{\mathbf x}^{(K_0+1)}$.\label{step:global_average}
}
\end{algorithmic}
\end{multicols}
\end{algorithm}

The proposed GenQSGD is presented in Algorithm~\ref{Alg:GenQSGD}.
During the $k_0$-th global iteration (Step~\ref{Step:begin_global_iteration}-Step~\ref{Step:end_global_iteration}), each worker $n\in\mathcal N$ initializes its local model based on the global model at the $(k_0-1)$-th global iteration and the quantized global model update (Step~\ref{Step:initialize_local_model}),
executes $K_n$ local iterations of the mini-batch SGD algorithm with mini-batch size $B$ and constant step size $\gamma^{(k_0)}$ (Step~\ref{Step:begin_local_iteration}-Step~\ref{Step:end_local_iteration}),
and sends the quantized overall local model update to the server (Step~\ref{Step:upload});
the server aggregates and computes the average of the quantized overall local model updates,
i.e., global model update, and sends the quantized result to all $N$ workers (Step~\ref{Step:broadcast}).

\begin{Rmk}[Generality of GenQSGD]
GenQSGD is general in the sense that $\mathbf K$, $B$, and $\mathbf\Gamma$ can be flexibly chosen,
and it includes some existing algorithms as special cases.
For the sake of discussion,
we let $s=\infty$ present the case without quantization.
In particular, GenQSGD with $K_n=1,n\in\mathcal N$ for $s_n=\infty,n\in\setNbar$ reduces to PM-SGD \cite{PMSGD};
and GenQSGD with $K_n=l\frac{I_n}{B},l\in\mathbb Z_+,n\in\mathcal N$ for $s_n=\infty,n\in\setNbar$ reduces to FedAvg \cite{FedAvg}.
\end{Rmk}

\subsection{Convergence Analysis}\label{SubSec:Conv}
In the rest of this paper, we assume that the following typical assumptions are satisfied \cite{YuHao}.
\begin{Asump}[I.I.D. Samples]\label{Asump:IID}
$\zeta_n,n\in\mathcal N$ are I.I.D..
\end{Asump}
\begin{Asump}[Smoothness]\label{Asump:Smoothness}
For all $n\in\mathcal N$, $f_n({\bf x})$ is continuously differentiable, and its gradient is Lipschitz continuous, i.e., there exists a constant $L>0$ such that $\left\|\nabla{f_n({\mathbf x})}-\nabla{f_n({\mathbf y})}\right\|_2 \leq L\left\|{\mathbf x}-{\mathbf y}\right\|_2$, for all $\mathbf x,\mathbf y\in\mathbb R^D$.
\end{Asump}
\begin{Asump}[Bounded Variances]\label{Asump:BoundedVariances}
For all $n\in\mathcal N$, there exists a constant $\sigma>0$ such that $\mathbb E\left[\left\|\nabla{F\left({\mathbf x};\zeta_n\right)}-\nabla{f_n({\mathbf x})}\right\|_2^2\right]\leq{\sigma^2}$, for all $\mathbf x\in\mathbb R^D$.
\end{Asump}
\begin{Asump}[Bounded Second Moments]\label{Asump:BoundedSecondMoments}
For all $n\in\mathcal N$, there exists a constant $G>0$ such that $\mathbb E\left[\left\|\nabla{F\left({\mathbf x};\zeta_n\right)}\right\|_2^2\right] \leq {G^2}$, for all $ \mathbf x\in\mathbb R^D$.
\end{Asump}
For notational simplicity, we
denote $\mathcal K_{\max}\triangleq\left\{1,2,\cdots,\max_{n\in\mathcal N}K_n\right\}$ and define:
\begin{align}
\tilde{\mathbf x}_n^{(k_0,k)}&\triangleq\left\{
    \begin{array}{ll}
    {\bf x}_n^{(k_0,k_n)},&k\in\mathcal K_n\\
    {\bf x}_n^{(k_0,K_n)},&k\in\mathcal K_{\max}\setminus\mathcal K_n
    \end{array},
\right.\ k_0\in\mathcal K_0,\ k\in\mathcal K_{\max},\ n\in\mathcal N,\label{eq:local_update_redefine}\\
\bar{\mathbf x}^{(k_0,k)}&\triangleq\frac{1}{N}\sum_{n\in\mathcal N}\tilde{\mathbf x}_n^{(k_0,k)},\ k_0\in\mathcal K_0,\ k\in\mathcal K_{\max},\label{eq:x_avg}\\
N_k&\triangleq\sum_{n\in\mathcal N}\mathbb I\left[k\leq K_n\right],\ k\in\mathcal K_{\max},\label{eq:N_k}
\end{align}
The goal is to synchronize the local iterations by letting each worker $n\in\mathcal N$ with $K_n<\max_{n\in\mathcal N}K_n$ run extra $\max_{n\in\mathcal N}K_n-K_n$ virtual local updates without changing its local model \cite{YuHao}.
Thus, $\bar{\mathbf x}^{(k_0,k)}$ in \eqref{eq:x_avg} can be interpreted as the average of the local models at the $k$-th synchronized local iteration within the $k_0$-th global iteration,
and $N_k$ in \eqref{eq:N_k} can be viewed as the number of workers conducting true local updates at the $k$-th synchronized local iteration within each global iteration.
The convergence of GenQSGD is summarized below.
\begin{Thm}[Convergence]\label{Thm:Convergence}
Suppose that Assumptions~\ref{Asump:Quantization},\ref{Asump:IID},\ref{Asump:Smoothness},\ref{Asump:BoundedVariances},\ref{Asump:BoundedSecondMoments} are satisfied and the step size $\gamma^{(k_0)}\in\left(0,\frac{1}{L}\right]$ for all $k_0\in\mathcal K_0$.
Then, for all $\mathbf K\in\mathbb Z_+^{N+1}$ and $B\in\mathbb Z_+$,
$\left\{\bar{\mathbf x}^{(k_0,k)}:k_0\in\mathcal K_0,k\in\mathcal K_{\max}\right\}$
generated by GenQSGD satisfies:\footnote{Following the convention in literature \cite{YuHao}, we use the expected squared gradient norm to characterize the convergence performance for a general ML problem.}
\begin{align}
&\frac{\sum_{k_0\in\mathcal K_0}\gamma^{(k_0)}\sum_{k\in\mathcal K_{\max}}\frac{N_k}{N}\mathbb E\left[\left\|\nabla f\left(\bar{\mathbf x}^{(k_0, k-1)}\right)\right\|^2\right]}{\sum_{k_0\in\mathcal K_0}\gamma^{(k_0)}\sum_{k\in\mathcal K_{\max}}\frac{N_k}{N}}\leq C_A(\mathbf K,B,\mathbf\Gamma),\nonumber
\end{align}
where
\begin{align}\label{eq:Convergence_RHS}
C_A(\mathbf K,B,\mathbf\Gamma)
\triangleq&\frac{c_1}{\sum_{n\in\mathcal N}K_n\sum_{k_0\in\mathcal K_0}\gamma^{(k_0)}}
+\frac{c_2\max_{n\in\mathcal N}K_n^2\sum_{k_0\in\mathcal K_0}\left(\gamma^{(k_0)}\right)^3}{\sum_{k_0\in\mathcal K_0}\gamma^{(k_0)}}\nonumber\\[-5pt]
&+\frac{c_3\sum_{k_0\in\mathcal K_0}\left(\gamma^{(k_0)}\right)^2}{B\sum_{k_0\in\mathcal K_0}\gamma^{(k_0)}}
+\frac{c_4\sum_{n\in\mathcal N}q_{s_0,s_n}K_n^2\sum_{k_0\in\mathcal K_0}\left(\gamma^{(k_0)}\right)^2}{\sum_{n\in\mathcal N}K_n\sum_{k_0\in\mathcal K_0}\gamma^{(k_0)}}.
\end{align}
Here, $q_{s_0,s_n}\triangleq q_{s_0}+q_{s_n}+q_{s_0}q_{s_n}$,
$c_1\triangleq2N\left(f\left(\hat{\mathbf x}^{(1)}\right)-f^*\right)$,
$c_2\triangleq4G^2L^2$,
$c_3\triangleq\frac{L\sigma^2}{N}$,
$c_4=2LG^2$,
and $f^*$ is the optimal value of Problem~\ref{Prob1}.
\end{Thm}
\begin{IEEEproof}
See Appendix~A.
\end{IEEEproof}
\begin{Rmk}[{\spaceskip=3.4pt\relax Generality of Convergence of GenQSGD}]
{\spaceskip=3.4pt\relax Theorem~\ref{Thm:Convergence} for the convergence of GenQSGD with $s_n\!=\!\infty,n\!\in\!\setNbar\!$ and $B\!=\!1$ reduces to the convergence of PR-SGD in \cite[Theorem 3]{YuHao}.}\footnote{The convergence error of PR-SGD \cite{{YuHao}} depends on the bound on the variance of the stochastic gradient and does not directly rely on the mini-batch size.}
\end{Rmk}

Theorem~\ref{Thm:Convergence} indicates that the convergence of GenQSGD is influenced by the algorithm parameters $\mathbf K$, $B$, $\mathbf\Gamma$ and the quantization parameters $s_n,n\in\bar{\mathcal N}$.
In the following, we illustrate how the four terms on the R.H.S of \eqref{eq:Convergence_RHS} change with $\mathbf K$, $B$, $\mathbf\Gamma$, and $s_n,n\in\bar{\mathcal N}$.
The first term decreases with $K_n$ for all $n\in\mathcal N$ and decreases with $\gamma^{(k_0)}$ for all $k_0\in\mathcal K_0$;
the second term increases with $\max_{n\in\mathcal N}K_n$;
the third term decreases with $B$ due to the decrease of the variance of a stochastic gradient;
and the last term increases with $q_{s_n}$ (decreases with $s_n$) and vanishes as $s_n\rightarrow\infty$ for all $n\in\setNbar$.

Next, we introduce three commonly adopted step size rules, supported by Tensorflow \cite{TensorFlow} and PyTorch \cite{PyTorch}, and specify
the convergence results under the three step size rules from Theorem~\ref{Thm:Convergence}.
\begin{itemize}
\item   Constant Step Size Rule:
For any $\gamma_C\in\left(0,\frac{1}{L}\right]$,
$\mathbf\Gamma$ satisfies:
\begin{align}\label{eq:constant_rule}
\mathbf\Gamma=\gamma_C\mathbf1.
\end{align}
\begin{Lem}[Convergence under Constant Step Size Rule]\label{Lem:Convergence_cons}
If $\mathbf\Gamma$ satisfies \eqref{eq:constant_rule}, then
\begin{align}\label{eq:Convergence_cons}
C_A(\mathbf K,B,\mathbf\Gamma)
=&\frac{c_1}{\gamma_C K_0\sum_{n\in\mathcal N}K_n}
\!+\!c_2\gamma_C^2\max_{n\in\mathcal N}K_n^2
\!+\!\frac{c_3\gamma_C}{B}
\!+\!\frac{c_4\gamma_C\sum_{n\in\mathcal N}q_{s_0,s_n}K_n^2}{\sum_{n\in\mathcal N}K_n}\nonumber\\
\triangleq&C_C(\mathbf K,B,\mathbf\Gamma),
\end{align}
and $C_C(\mathbf K,B,\mathbf\Gamma)
\rightarrow c_2\gamma_C^2\max_{n\in\mathcal N}K_n^2
+\frac{c_3\gamma_C}{B}+\frac{c_4\gamma_C\sum_{n\in\mathcal N}q_{s_0,s_n}K_n^2}{\sum_{n\in\mathcal N}K_n}$ as $K_0\rightarrow\infty$.
Furthermore, if $\mathbf\Gamma$ satisfies \eqref{eq:constant_rule} with $\gamma_C=\frac{\sqrt{N}}{L\sqrt{K_0\bar K}}$ and $K_n=\bar K$, $q_{s_0,s_n}=\frac{1}{N\bar K}$, $n\in\mathcal N$ with $\bar K\leq\frac{\left(K_0\bar K\right)^{1/4}}{N^{3/4}}$, then $C_A(\mathbf K,B,\mathbf\Gamma)=\mathcal O\left(K_0^{-\frac{1}{2}}\right)$.
\end{Lem}
\begin{IEEEproof}
We readily show \eqref{eq:Convergence_cons} by substituting \eqref{eq:constant_rule} into \eqref{eq:Convergence_RHS}.
Then, the limit of $C_C(\mathbf K,B,\mathbf\Gamma)$ can be easily derived.
By substituting $\mathbf\Gamma$ given by \eqref{eq:constant_rule} with $\gamma_C=\frac{\sqrt{N}}{L\sqrt{K_0\bar K}}$ and $K_n=\bar K$, $q_{s_0,s_n}=\frac{1}{N\bar K}$, $n\in\mathcal N$ with $\bar K\leq\frac{\left(K_0\bar K\right)^{1/4}}{N^{3/4}}$ into \eqref{eq:Convergence_RHS},
we have
$C_A(\mathbf K,B,\mathbf\Gamma)
=\frac{2L\left(f\left(\hat{\mathbf x}^{(1)}\right)-f^*\right)}{\sqrt{NK_0\bar K}}
+\frac{c_2N\bar K}{L^2K_0}
+\frac{\sigma^2}{B\sqrt{NK_0\bar K}}
+\frac{c_4}{L\sqrt{NK_0\bar K}}
{\overset{(a)}{\leq}}\frac{2L\left(f\left(\hat{\mathbf x}^{(1)}\right)-f^*\right)}{\sqrt{NK_0\bar K}}
+\frac{c_2}{L^2\sqrt{NK_0\bar K}}
+\frac{\sigma^2}{B\sqrt{NK_0\bar K}}
+\frac{c_4}{L\sqrt{NK_0\bar K}}$, where (a) is due to $\bar K\leq\frac{\left(K_0\bar K\right)^{1/4}}{N^{3/4}}$.
Thus, we have $C_A(\mathbf K,B,\mathbf\Gamma)=\mathcal O\left(K_0^{-\frac{1}{2}}\right)$.
Therefore, we can show Lemma~\ref{Lem:Convergence_cons}.\end{IEEEproof}
\item   Exponential Step Size Rule:
For any $\gamma_E\in\left(0,\frac{1}{L}\right]$ and $\rho_E\in(0,1)$, $\mathbf\Gamma$ satisfies:
\begin{align}\label{eq:exp_rule}
\gamma^{(k_0)}=\rho_E^{k_0-1}\gamma_E,\ k_0\in\mathcal K_0.
\end{align}
\begin{Lem}[Convergence under Exponential Step Size Rule]\label{Lem:Convergence_exp}
If $\mathbf\Gamma$ satisfies \eqref{eq:exp_rule}, then
\begin{align}\label{eq:Convergence_exp}
C_A(\mathbf K,B,\mathbf\Gamma)
=&\frac{a_1c_1}{\left(1-\rho_E^{K_0}\right)\sum_{n\in\mathcal N}K_n}
+\frac{a_2c_2\left(1-\rho_E^{3K_0}\right)\max_{n\in\mathcal N}K_n^2}
{\left(1-\rho_E^{K_0}\right)}\nonumber\\[-5pt]
&+\frac{a_3\left(1-\rho_E^{2K_0}\right)}{\left(1-\rho_E^{K_0}\right)}
\left(\frac{c_3}{B}+\frac{c_4\sum_{n\in\mathcal N}q_{s_0,s_n}K_n^2}{\sum_{n\in\mathcal N}K_n}\right)
\triangleq C_E(\mathbf K,B,\mathbf\Gamma),
\end{align}
and $C_E(\mathbf K,B,\mathbf\Gamma)
\rightarrow\frac{a_1c_1}{\sum_{n\in\mathcal N}K_n}
+a_2c_2\max_{n\in\mathcal N}K_n^2\!+\!\frac{a_3c_3}{B}
+\frac{a_3c_4\sum_{n\in\mathcal N}q_{s_0,s_n}K_n^2}{\sum_{n\in\mathcal N}K_n}$
as $K_0\!\rightarrow\!\infty$, where $a_1\!\triangleq\!\frac{1-\rho_E}{\gamma_E}$,
$a_2\!\triangleq\!\frac{\gamma_E^2}{1+\rho_E+\rho_E^2}$,
and $a_3\!\triangleq\!\frac{\gamma_E}{1+\rho_E}$.
Furthermore, if $\mathbf\Gamma$ satisfies \eqref{eq:exp_rule} with $\gamma_E=\frac{\sqrt{N}}{L\sqrt{K_0\bar K}}$ and $K_n\!=\!\bar K$, $q_{s_0,s_n}=\frac{1}{N\bar K}$, $n\!\in\!\mathcal N$ with $\bar K\!\leq\!\frac{\left(K_0\bar K\right)^{1/4}}{N^{3/4}}$,
then $C_A(\mathbf K,B,\mathbf\Gamma)=\mathcal O\left(K_0^{-\frac{1}{2}}\right)$.
\end{Lem}
\begin{IEEEproof}
See Appendix B.
\end{IEEEproof}
\item   Diminishing Step Size Rule:
The step size sequence $\mathbf\Gamma$ for diminishing step size rule satisfies:
\begin{align}\label{eq:dim_rule}
\sum_{k_0\in\mathcal K_0}\gamma^{(k_0)}\rightarrow\infty,\sum_{k_0\in\mathcal K_0}\left(\gamma^{(k_0)}\right)^2\rightarrow0,
\end{align}
as $K_0\rightarrow\infty$.
We further consider the following widely used step size sequence $\mathbf\Gamma$ satisfying the diminishing step size rule in \eqref{eq:dim_rule}:
For any $\gamma_D\in\left(0,\frac{1}{L}\right]$ and $\rho_D\in\mathbb R_+$,
\begin{align}\label{eq:dim_rule_sp}
\gamma^{(k_0)}=\frac{\gamma_D}{1+\frac{k_0}{\rho_D}}
=\frac{\rho_D\gamma_D}{k_0+\rho_D},\ k_0\in\mathcal K_0.
\end{align}
\begin{Lem}[Convergence under Diminishing Step Size Rule]\label{Lem:Convergence_dim}
If $\mathbf\Gamma$ satisfies \eqref{eq:dim_rule_sp}, then\footnote{A tighter bound on $C_A(\mathbf K,B,\mathbf\Gamma)$ can be found in Appendix C.
We use the upper bound $C_D(\mathbf K,B,\mathbf\Gamma)$ in \begin{tiny}\eqref{eq:Convergence_dim}\end{tiny} in the subsequent optimization for tractability.}
\begin{align}\label{eq:Convergence_dim}
C_A(\mathbf K,B,\mathbf\Gamma)
<&\frac{b_1c_1}{\ln\!\left(\frac{K_0+\rho_D+1}{\rho_D+1}\right)\!\!\sum_{n\in\mathcal N}K_n}
\!+\!\frac{b_2c_2\max_{n\in\mathcal N}K_n^2}{\ln\!\left(\frac{K_0+\rho_D+1}{\rho_D+1}\right)}
\!+\!\frac{b_3c_3}{B\ln\!\left(\frac{K_0+\rho_D+1}{\rho_D+1}\right)}\nonumber\\
&\!+\!\frac{b_3c_4\sum_{n\in\mathcal N}q_{s_0,s_n}K_n^2}{\ln\!\left(\frac{K_0+\rho_D+1}{\rho_D+1}\right)\!\!\sum_{n\in\mathcal N}K_n}
\triangleq C_D(\mathbf K,B,\mathbf\Gamma),
\end{align}
and for all $\mathbf\Gamma$ satisfying \eqref{eq:dim_rule},
$C_D(\mathbf K,B,\mathbf\Gamma)
\rightarrow0$
as $K_0\rightarrow\infty$, where $b_1\triangleq\frac{1}{\rho_D\gamma_D}$,
$b_2\!\triangleq\!\frac{\rho_D^2\gamma_D^2}{\left(\rho_D+1\right)^3}
\!+\!\frac{\rho_D^2\gamma_D^2}{2\left(\rho_D+1\right)^2}$,
and $b_3\!\triangleq\!\frac{\rho_D\gamma_D}{\left(\rho_D+1\right)^2}
\!+\!\frac{\rho_D\gamma_D}{\rho_D+1}$.
Furthermore, if $\mathbf\Gamma$ satisfies \eqref{eq:dim_rule_sp} with $\gamma_D\!=\!\frac{\sqrt{N}}{L\sqrt{K_0\bar K}}$ and $K_n\!=\!\bar K$, $q_{s_0,s_n}\!=\!\frac{1}{N\bar K}$, $n\!\in\!\mathcal N$ with $\bar K\!\leq\!\frac{\left(K_0\bar K\right)^{1/4}}{N^{3/4}}$,
then $C_A(\mathbf K,B,\mathbf\Gamma)=\mathcal O\left(K_0^{-\frac{1}{2}}\right)$.
\end{Lem}
\begin{IEEEproof}
See Appendix C.
\end{IEEEproof}
\end{itemize}

For any given $\gamma_E,\gamma_C\in\left(0,\frac{1}{L}\right]$ with $\gamma_E=\gamma_C$, as $\rho_E\rightarrow1$, the exponential step size rule given by \eqref{eq:exp_rule} approaches the constant step size rule given by \eqref{eq:constant_rule}, and the result in Lemma~\ref{Lem:Convergence_exp} approaches that in Lemma~\ref{Lem:Convergence_cons} correspondingly.
Lemma~\ref{Lem:Convergence_cons} and Lemma~\ref{Lem:Convergence_exp} indicate that for the constant and exponential step size rules, $\left\{\bar{\mathbf x}^{(k_0,k)}:k_0\in\mathcal K_0,k\in\mathcal K_{\max}\right\}$
generated by GenQSGD is guaranteed to converge to within some range of a stationary point of Problem~\ref{Prob1}, as $K_0\rightarrow\infty$.
In contrast, Lemma~\ref{Lem:Convergence_dim} indicates that for the diminishing step size rule, $\left\{\bar{\mathbf x}^{(k_0,k)}:k_0\in\mathcal K_0,k\in\mathcal K_{\max}\right\}$ generated by GenQSGD is guaranteed to converge to a stationary point of Problem~\ref{Prob1}, as $K_0\rightarrow\infty$.
Besides, Lemmas~\ref{Lem:Convergence_cons}--\ref{Lem:Convergence_dim} indicate that GenQSGD under the constant, exponential, and diminishing step size rules yield the same convergence rate in order.\footnote{The constant and exponential step size rules usually yield faster convergence than the diminishing step size rule when the step size parameters are properly chosen.
Besides, the exponential and diminishing step size rules generally lead to more robust convergence speeds against the choice of step size parameters.}

\section{Performance Metrics and Typical Formulations}\label{Sec:PerformanceMetrics}
In this section, we first introduce performance metrics for implementing GenQSGD in the edge computing system.
Then, we briefly discuss typical formulations for optimizing the overall implementing process of GenQSGD.\footnote{In practice, one model update usually contains a large number of information bits due to the high model dimension, and hence its transmission lasts several time slots (with different channel states).
Therefore, we consider the average rates of the server and workers in optimizing the overall implementing process of GenQSGD.
Note that most existing work on short time-scale optimal computing and communication resource allocation relies on the assumption that the transmission of one model update can be completed within one time slot, which may not be reasonable for a high dimensional ML model \cite{ResAlls1-1,ResAlls2-2,ResAlls2-3,ResAlls3-1,MIMO}.
}

\subsection{Performance Metrics}
Let $C_n$ denote the number of CPU-cycles required for worker $n\in\mathcal N$ to compute $\nabla F({\mathbf x};\xi_i)$ for all $\mathbf x\in\mathbb R^{D}$ and $i\in\mathcal I_n$,
and let $C_0$ denote the number of CPU-cycles required for the server to compute one global model update.
Then, within each global iteration,
the computation time for local model training at the $N$ workers is $B\max_{n\in\mathcal N}\frac{K_nC_n}{F_n}$, as the local iterations are executed in a parallel manner;
the computation time for global averaging at the server is $\frac{C_0}{F_0}$;
the communication time for all $N$ workers to send their quantized local model updates to the server is $\max_{n\in\mathcal N}\frac{M_{s_n}}{r_n}$, as the quantized messages are transmitted at given transmission rates using FDMA;
and the communication time for the server to multicast the quantized global model update to all $N$ workers is $\frac{M_{s_0}}{r_0}$.
Thus, the overall time for implementing GenQSGD, referred to as the time cost, is given by:
\begin{align}\label{eq:Time}
&T(\mathbf K,B)=K_0\left(B\max_{n\in\mathcal N}\frac{C_n}{F_n}K_n+\frac{C_0}{F_0}+\max_{n\in\mathcal N}\frac{M_{s_n}}{r_n}+\frac{M_{s_0}}{r_0}\right).\setlength{\belowdisplayskip}{-5pt}
\end{align}
Let $\alpha_0$ and $\alpha_n$ denote the constant factors determined by the switched capacitances of the server and worker $n\in\mathcal N$, respectively \cite{capacitance}.
Then, within each global iteration,
the computation energy for local model training at the $N$ workers is $B\sum_{n\in\mathcal N}\alpha_nK_nC_n{F_n}^2$;
the computation energy for global averaging at the server is $\alpha_0 C_0{F_0}^2$;
the communication energy for all $N$ workers to send their quantized local model updates to the server is $\sum_{n\in\mathcal N}\frac{p_nM_{s_n}}{r_n}$;
and the communication energy for the server to multicast the quantized global model update to all $N$ workers is $\frac{p_0M_{s_0}}{r_0}$.
Thus, the overall energy for implementing GenQSGD, referred to as the energy cost, is given by:
\begin{align}\label{eq:Energy}
&E(\mathbf K,B)=K_0\left(B\sum_{n\in\mathcal N}\alpha_nC_nF_n^2K_n+\alpha_0C_0F_0^2+\sum_{n\in\setNbar}\frac{p_nM_{s_n}}{r_n}\right).
\end{align}

In this paper, we assume $L$, $\sigma$, $G$, and a lower bound\footnote{In practice, $f^*$ is usually unknown before solving Problem~\ref{Prob1}.
However, in some cases, we can bound $f^*$ from below.
In the following, we also use $c_1$ to represent the corresponding upper bound on $2N\left(f\left(\hat{\mathbf x}^{(1)}\right)-f^*\right)$, with slight abuse of notation.} on $f^*$ (equivalently, the values of $c_1$, $c_2$, $c_3$, and $c_4$) can be obtained by the server from pre-training based on the data samples stored on the server.
We use $C_A(\mathbf K,B,\mathbf\Gamma)$ given by \eqref{eq:Convergence_RHS}, $C_C(\mathbf K,B,\mathbf\Gamma)$ given by \eqref{eq:Convergence_cons}, $C_E(\mathbf K,B,\mathbf\Gamma)$ given by \eqref{eq:Convergence_exp}, and $C_D(\mathbf K,B,\mathbf\Gamma)$ given by \eqref{eq:Convergence_dim} as the measures of convergence errors of GenQSGD for arbitrary, constant, exponential, and diminishing step size rules, respectively.

\subsection{Typical Formulations}
In this subsection, we discuss formulations for optimally balancing among the three performance metrics, i.e., time cost, energy cost, and convergence error.
Specifically, we impose constraints on the performance metrics with firm requirements (if they exist) and minimize the rest\footnote{
If there are more than one performance metrics to be minimized, we minimize the weighted sum of them.
If all three performance metrics have to satisfy firm requirements, we formulate a feasibility problem with the goal of finding any variables that satisfy the corresponding constraints.}
by optimizing the algorithm parameters $\mathbf K\in\mathbb Z_+^{N+1}$, $B\in\mathbb Z_+$, and $\mathbf\Gamma$ satisfying:
\begin{align}\label{eq:Cons_gamma}
\mathbf0\prec\mathbf\Gamma\preceq\frac{1}{L}\mathbf1.
\end{align}
For tractability, we relax the integer constraints, $\mathbf K\in\mathbb Z_+^{N+1}$ and $B\in\mathbb Z_+$, to their continuous counterparts, $\mathbf K\succ\mathbf 0$ and $B>0$, respectively.
Note that a nearly optimal point satisfying the original integer constraints can be easily constructed based on an optimal point of a relaxed problem.
In this paper, as an example, we focus on minimizing the energy cost $E(\mathbf K,B)$ subject to the following time constraint and convergence error constraint for any given $m\in\{A,C,E,D\}$:
\begin{align}
&T(\mathbf K,B)\leq T_{\max},\label{eq:Cons_time}\\
&C_m(\mathbf K,B,\mathbf\Gamma)\leq C_{\max},\label{eq:Cons_conv}
\end{align}
where $T_{\max}$ and $C_{\max}$ denote the limits on the time cost and convergence error, respectively.
Specifically, in Sec.~\ref{sec:Opt_fix_step_size}, we optimize $\mathbf K$ and $B$ for any fixed $\mathbf\Gamma$ satisfying \eqref{eq:constant_rule}, \eqref{eq:exp_rule}, and \eqref{eq:dim_rule_sp}, respectively;
and in Sec.~\ref{sec:Opt_adj_step_size}, we jointly optimize $\mathbf K$, $B$, and $\mathbf\Gamma$.
The optimization problems can be solved by the server (in an offline manner) before the implementation of GenQSGD.\footnote{Recall that the parameters of the expected risk function $f$, i.e., $L$, $\sigma$, $G$, and a lower bound of $f^*$ can be obtained by the server from pre-training based on the data samples stored on the server;
the system parameters, i.e., $F_n,p_n,r_n,s_n,n\in\setNbar$, are known to the server;
and in applications with preset step size sequences, $\mathbf\Gamma$ has been chosen by the server.}
We shall see that the proposed optimization framework is applicable for the other problems mentioned above and applies to the parameter optimization of existing FL algorithms, such as PM-SGD, FedAvg, and PR-SGD, by fixing some algorithm parameters and optimizing the others.

\section{Optimization of Algorithm Parameters for Fixed Step Size Sequences}\label{sec:Opt_fix_step_size}
In many current applications, the step size sequences are treated as hyperparameters and chosen by the server in advance to reduce the convergence error and to improve the convergence speed of centralized learning.
Thus, in this section, we optimize the global and local iteration numbers and mini-batch size for a fixed step size sequence under each of the three step size rules to minimize the energy cost under the time constraint and convergence error constraint,
as illustrated in Fig.~\ref{Fig:Structure_1}.
\begin{figure}[t]
\setlength{\belowcaptionskip}{-10pt}
\begin{center}
{\includegraphics[width=300pt]{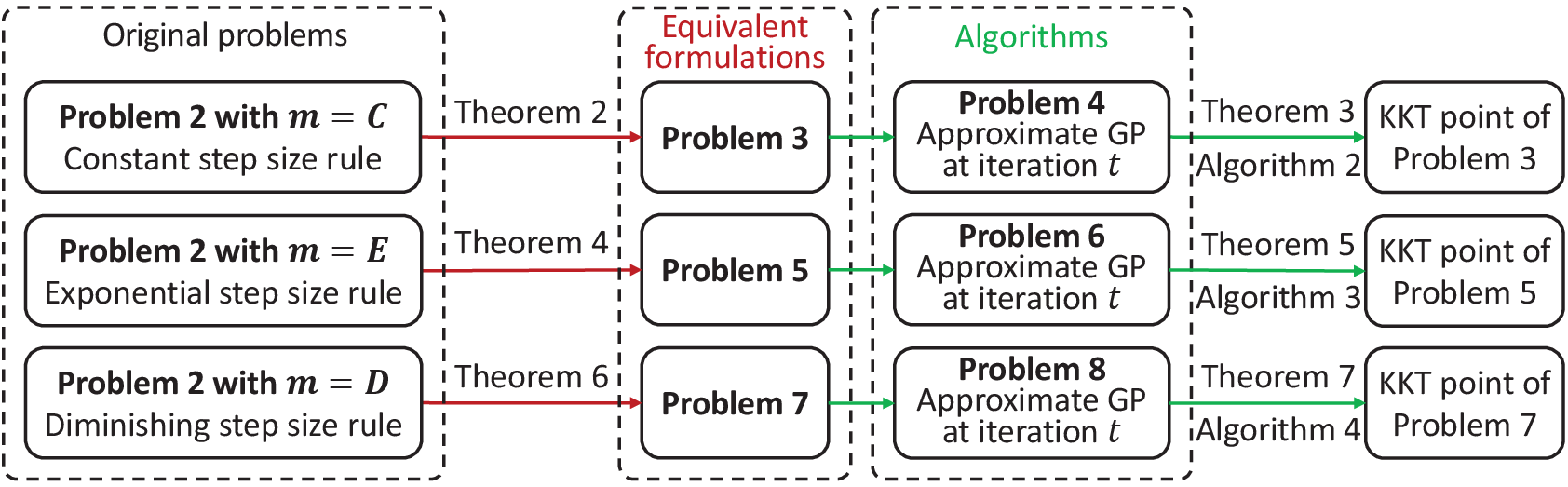}}
\caption{\small{Proposed solutions to optimization of algorithm parameters for fixed step size sequences.}}\label{Fig:Structure_1}
\end{center}
\end{figure}

\subsection{Problem Formulation}\label{SubSec:Formulation_Fixed_step_size}
For any fixed $\mathbf\Gamma$ satisfying \eqref{eq:constant_rule}, \eqref{eq:exp_rule}, and \eqref{eq:dim_rule_sp},
i.e., for any $m\in\{C,E,D\}$,
we optimize $\mathbf K$ and $B$ to minimize $E(\mathbf K,B)$ under the constraints in \eqref{eq:Cons_time} and \eqref{eq:Cons_conv}.
\begin{Prob}[Optimization of Global and Local Iteration Numbers and Mini-batch Size for Fixed Step Size Sequences]\label{Prob:fixed_step_size}
For any given $m\in\{C,E,D\}$,
\begin{align}
\min_{\mathbf K\succ\mathbf 0, B>0}&{\quad}E(\mathbf K,B)\nonumber\\
\mathrm{s.t.}&{\quad}\eqref{eq:Cons_time},\ \eqref{eq:Cons_conv}.\nonumber
\end{align}
\end{Prob}
Apparently, the minimum energy cost (i.e., the optimal value of Problem~\ref{Prob:fixed_step_size}) does not increase (usually decreases) with the limits on the time cost and convergence error, i.e., $T_{\max}$ and $C_{\max}$, indicating an optimal trade-off among the energy cost, time cost, and convergence error.
The constraints in \eqref{eq:Cons_time} and \eqref{eq:Cons_conv} are non-convex and contain non-differentiable functions, and the objective function is non-convex.
Thus, Problem~\ref{Prob:fixed_step_size} with $m=C,E,D$ are challenging non-convex problems with non-differentiable constraint functions.\footnote{The goal of solving a constrained non-convex problem is usually to obtain a KKT point.}
In Sec.~\ref{SubSec:Solution_Fixed_constant}, Sec.~\ref{SubSec:Solution_Fixed_exp}, and Sec.~\ref{SubSec:Solution_Fixed_dim}, we will develop algorithms for solving Problem~\ref{Prob:fixed_step_size} with $m=C$, $m=E$, and $m=D$, respectively.

\subsection{Solution to Problem~\ref{Prob:fixed_step_size} with $m=C$}\label{SubSec:Solution_Fixed_constant}
To address the challenge caused by the non-differentiable constraint functions in \eqref{eq:Cons_time} and \eqref{eq:Cons_conv} with $m=C$,
we equivalently transform Problem~\ref{Prob:fixed_step_size} with $m=C$ into the following problem with differentiable constraint functions.
\begin{Prob}[Equivalent Problem of Problem~\ref{Prob:fixed_step_size} with $m=C$]\label{Prob:constant_eq}
For any given $\gamma_C\in\left(0,\frac{1}{L}\right]$,
\begin{align}
\min_{\substack{\mathbf K\succ\mathbf0,B,T_1,T_2>0}}
&{\quad}E(\mathbf K,B)\nonumber\\
\mathrm{s.t.}
&{\quad}\frac{C_n}{F_n}K_n T_1^{-1}\leq1,\ n\in\mathcal N,\label{eq:fix_epi_cons1}\\
&{\quad}K_nT_2^{-1}\leq1,\ n\in\mathcal N,\label{eq:fix_epi_cons2}\\
&{\quad}\left(\left(\frac{C_0}{F_0}+\max_{n\in\mathcal N}\frac{M_{s_n}}{r_n}+\frac{M_{s_0}}{r_0}\right)+BT_1\right)\frac{K_0}{T_{\max}}
\leq1,\label{eq:fix_epi_cons_T}\\
&{\quad}\frac{c_1}{\gamma_C K_0\sum_{n\in\mathcal N}K_n}+c_2\gamma_C^2T_2^2
+\frac{c_3\gamma_C}{B}
+\frac{c_4\gamma_C\sum_{n\in\mathcal N}q_{s_0,s_n}K_n^2}{\sum_{n\in\mathcal N}K_n}
\leq C_{\max}.\label{eq:Cons_conv_cons}
\end{align}
\end{Prob}
\begin{Thm}[Equivalence between Problem~\ref{Prob:fixed_step_size} with $m=C$ and Problem~\ref{Prob:constant_eq}]\label{Thm:Equivalent_Prob_cons}
If $\left(\mathbf K^*,B^*,T_1^*,T_2^*\right)$ is an optimal point of Problem~\ref{Prob:constant_eq},
then $\left(\mathbf K^*, B^*\right)$ is an optimal point of Problem~\ref{Prob:fixed_step_size} with $m=C$.
\end{Thm}
\begin{IEEEproof}
By introducing auxiliary variables $T_1,T_2>0$, replacing $\max_{n\in\mathcal N}\frac{C_n}{F_n}K_n$ in \eqref{eq:Cons_time} and $\max_{n\in\mathcal N}K_n^2$ in \eqref{eq:Cons_conv} with $m=C$ with $T_1$ and $T_2^2$, respectively, and adding the inequality constraints in \eqref{eq:fix_epi_cons1} and \eqref{eq:fix_epi_cons2},
we transform Problem~\ref{Prob:fixed_step_size} with $m=C$ into Problem~\ref{Prob:constant_eq}.
Suppose that there exists $(\mathbf K^\dag,B^\dag)\neq(\mathbf K^*,B^*)$, where $\mathbf K^\dag\triangleq(K_n^\dag)_{n\in\setNbar}$, satisfying all constraints of Problem~\ref{Prob:fixed_step_size} with $m=C$ and $E(\mathbf K^\dag,B^\dag)<E(\mathbf K^*,B^*)$.
Construct $\hat{T_1}=\max_{n\in\mathcal N}\frac{C_n}{F_n}K_n^\dag$ and $\hat T_2=\max_{n\in\mathcal N}K_n^\dag$.
Obviously, $(\mathbf K^\dag,B^\dag,\hat{T_1},\hat{T_2})$ satisfies all constraints of Problem~\ref{Prob:constant_eq} and $E(\mathbf K^\dag,B^\dag)<E(\mathbf K^*,B^*)$,
contradicting with the optimality of $(\mathbf K^*,B^*,T_1^*,T_2^*)$ for Problem~\ref{Prob:constant_eq}.
Thus, by contradiction, we can show that $\left(\mathbf K^*, B^*\right)$ is an optimal point of Problem~\ref{Prob:fixed_step_size} with $m=C$.
Therefore, we can show Theorem~\ref{Thm:Equivalent_Prob_cons}.
\end{IEEEproof}
By Theorem~\ref{Thm:Equivalent_Prob_cons}, we can solve Problem~\ref{Prob:constant_eq} instead of Problem~\ref{Prob:fixed_step_size} with $m=C$.
Note that $E(\mathbf K,B)$ and the constraint function in \eqref{eq:fix_epi_cons_T} are posynomials,
the constraint functions in \eqref{eq:fix_epi_cons1} and \eqref{eq:fix_epi_cons2} are monomials,
the first term of the constraint function in \eqref{eq:Cons_conv_cons} is a ratio between a constant and a posynomial,
and the last term of the constraint function in \eqref{eq:Cons_conv_cons} is a ratio between two posynomials.
Thus, Problem~\ref{Prob:constant_eq} is a non-convex problem that is more complicated than a CGP.

In the following, using GIA \cite{GIA} and tricks for solving CGP \cite{CGP}, we propose an iterative algorithm to obtain a KKT point of Problem~\ref{Prob:constant_eq}.
The idea is to construct and solve a sequence of successively refined approximate geometric programs (GPs).
Specifically, at iteration $t$,
update $\left(\mathbf K^{(t)}, B^{(t)}\right)$ by solving Problem~\ref{Prob:constant_GP},
which is parameterized by $\mathbf K^{(t-1)}$ obtained at iteration $t-1$.
\begin{algorithm}[t]
\caption{Algorithm for Obtaining a KKT Point of Problem~\ref{Prob:constant_eq}}
\label{Alg:constant}
\begin{algorithmic}[1]
\small{
\STATE {\bf{Initialize:}}
Choose any feasible solution $\left(\mathbf K^{(0)}, B^{(0)}, T_1^{(0)}, T_2^{(0)}\right)$ of Problem~\ref{Prob:constant_eq},
and set $t=1$.
\REPEAT
    \STATE
    Compute $\left(\mathbf K^{(t)}, B^{(t)}, T_1^{(t)}, T_2^{(t)}\right)$ by transforming Problem~\ref{Prob:constant_GP} into a GP in convex form
    and solving it with standard convex optimization techniques.
    \STATE Set $t:=t+1$.
\UNTIL{Some convergence criteria is met.}
}\normalsize
\end{algorithmic}
\end{algorithm}

\begin{Prob}[Approximate GP of Problem~\ref{Prob:constant_eq} at Iteration $t$]\label{Prob:constant_GP}
For any given $\gamma_C\in\left(0,\frac{1}{L}\right]$,
\begin{align}
&\min_{\substack{\mathbf K\succ\mathbf0,B,T_1,T_2>0}}{\quad}E(\mathbf K,B)\nonumber\\
&\mathrm{s.t.}{\quad}\eqref{eq:fix_epi_cons1},\ \eqref{eq:fix_epi_cons2},\
\eqref{eq:fix_epi_cons_T},\nonumber\\
&{\quad}\frac{c_1}{C_{\max}\gamma_C K_0\prod_{n\in\mathcal N}\left(\frac{K_n}{\beta_n^{(t-1)}}\right)^{\beta_n^{(t-1)}}}
+\frac{c_2\gamma_C^2T_2^2}{C_{\max}}
+\frac{c_3\gamma_C}{C_{\max}B}+\frac{c_4\gamma_C\sum_{n\in\mathcal N}q_{s_0,s_n}K_n^2}{C_{\max}\prod_{n\in\mathcal N}\left(\frac{K_n}{\beta_n^{(t-1)}}\right)^{\beta_n^{(t-1)}}}
\leq1,\label{eq:constant_cons_C_approx}
\end{align}
where $\beta_n^{(t-1)}\triangleq\frac{K_n^{(t-1)}}{\sum_{n\in\mathcal N}K_n^{(t-1)}}$
and $\left(\mathbf K^{(t)}, B^{(t)}, T_1^{(t)}, T_2^{(t)}\right)$ denotes an optimal solution of Problem~\ref{Prob:constant_GP}.
\end{Prob}
The constraint function in \eqref{eq:constant_cons_C_approx}, constructed by adopting a commonly used trick in CGP \cite[Lemma1]{CGP} that is based on the arithmetic-geometric mean inequality, is an approximation of the constraint function in \eqref{eq:Cons_conv_cons} at $\mathbf K^{(t-1)}$ and is a posynomial.
As a result, Problem~\ref{Prob:constant_GP} is a standard GP and can be readily transformed into a convex problem and solved by using standard convex optimization techniques such as interior-point methods. In particular, if an interior-point method is applied, the computational complexity for solving Problem~\ref{Prob:constant_GP} is $\mathcal O(N^{3.5})$ \cite{CVX}.
The details are summarized in Algorithm~\ref{Alg:constant}.
Following \cite[Proposition 3]{CGP}, we have the following result.
\begin{Thm}[Convergence of Algorithm~\ref{Alg:constant}]\label{Thm:constant_convergence}
$\left(\mathbf K^{(t)},B^{(t)},T_1^{(t)},T_2^{(t)}\right)$ obtained by Algorithm~\ref{Alg:constant} converges to a KKT point of Problem~\ref{Prob:constant_eq}, as $t\rightarrow\infty$.
\end{Thm}
\begin{IEEEproof}
By the arithmetic-geometric mean inequality, we know that $\sum_{n}\eta_nv_n\geq\prod_{n}v_n^{\eta_n}$, where $\mathbf v\succ0$ and $\boldsymbol\eta\succeq0$, $\mathbf 1^T\boldsymbol\eta=1$. Letting $K_n=\eta_nv_n$ and $\beta_n^{(t-1)}=\eta_n$ for all $n\in\mathcal N$, we have $\sum_{n\in\mathcal N}K_n\geq\prod_{n\in\mathcal N}\left(\frac{K_n}{\beta_n^{(t-1)}}\right)^{\beta_n^{(t-1)}}$, which indicates that \eqref{eq:constant_cons_C_approx} satisfies Property (i) in \cite{GIA}. By letting $\beta_n^{(t-1)}=\frac{K_n^{(t-1)}}{\sum_{n\in\mathcal N}K_n^{(t-1)}}$, we have $\sum_{n\in\mathcal N}K_n^{(t-1)}=\prod_{n\in\mathcal N}\left(\frac{K_n^{(t-1)}}{\beta_n^{(t-1)}}\right)^{\beta_n^{(t-1)}}$, which indicates that \eqref{eq:constant_cons_C_approx} satisfies Property (ii) in \cite{GIA}. Furthermore, we readily show that \eqref{eq:constant_cons_C_approx} satisfies Property (iii) in \cite{GIA} by taking the derivatives of $\sum_{n\in\mathcal N}K_n$ and $\prod_{n\in\mathcal N}\left(\frac{K_n}{\beta_n^{(t-1)}}\right)^{\beta_n^{(t-1)}}$ with respect to $K_n,n\in\mathcal N$.
Besides, as the convex form of Problem~\ref{Prob:constant_GP} (which is a GP) satisfies the Slater's condition \cite{CGP}, Problem~\ref{Prob:constant_GP} has a zero duality gap.
Thus, all the conditions for \cite[Theorem 1]{GIA} are satisfied.
By \cite[Theorem 1]{GIA}, Theorem~\ref{Thm:constant_convergence} is readily shown.
\end{IEEEproof}

\subsection{Solution to Problem~\ref{Prob:fixed_step_size} with $m=E$}\label{SubSec:Solution_Fixed_exp}
Compared to Problem~\ref{Prob:fixed_step_size} with $m=C$, Problem~\ref{Prob:fixed_step_size} with $m=E$ has an extra challenge caused by the products of exponential functions and posynomials.
To address the challenge caused by the non-differentiable constraint functions in \eqref{eq:Cons_time} and \eqref{eq:Cons_conv} with $m=E$ as well as the extra challenge,
we equivalently transform Problem~\ref{Prob:fixed_step_size} with $m=E$ into the following problem.
\begin{Prob}[Equivalent Problem of Problem~\ref{Prob:fixed_step_size} with $m=E$]\label{Prob:exp_eq}
For any given $\gamma_E\in\left(0,\frac{1}{L}\right]$ and $\rho_E\in(0,1)$,
\begin{align}
&\min_{\substack{\mathbf K\succ\mathbf0,B,T_1,T_2,X_0>0}}{\quad}
E(\mathbf K,B)\nonumber\\
\mathrm{s.t.}&{\quad}
\eqref{eq:fix_epi_cons1},\ \eqref{eq:fix_epi_cons2},\ \eqref{eq:fix_epi_cons_T},\nonumber\\
&{\quad}\frac{a_1c_1+\left(a_2c_2T_2^2+\frac{a_3c_3}{B}+C_{\max}X_0\right)\sum_{n\in\mathcal N} K_n+a_3c_4\sum_{n\in\mathcal N} q_{s_0,s_n}K_n^2}
{\left(C_{\max}+a_2c_2T_2^2X_0^3+\frac{a_3c_3X_0^2}{B}\right)\sum_{n\in\mathcal N} K_n+a_3c_4X_0^2\sum_{n\in\mathcal N} q_{s_0,s_n}K_n^2}\leq1,\label{eq:exp_cons_C}\\
&{\quad}X_0\ln\frac{1}{X_0}\leq X_0K_0\ln\frac{1}{\rho_E},\label{eq:exp_cons_X0_1}\\
&{\quad}K_0\ln\frac{1}{\rho_E}\leq\ln\frac{1}{X_0},\label{eq:exp_cons_X0_2}\\
&{\quad}X_0<1.\label{eq:cons_X_0}
\end{align}
\end{Prob}

\begin{Thm}[Equivalence between Problem~\ref{Prob:fixed_step_size} with $m=E$ and Problem~\ref{Prob:exp_eq}]\label{Thm:Equivalent_Prob_exp}
If $\left(\mathbf K^*,B^*,T_1^*,T_2^*,X_0^*\right)$ is an optimal point of Problem~\ref{Prob:exp_eq},
then $\left(\mathbf K^*, B^*\right)$ is an optimal point of Problem~\ref{Prob:fixed_step_size} with $m=E$.
\end{Thm}
\begin{IEEEproof}
See Appendix D.
\end{IEEEproof}

By Theorem~\ref{Thm:Equivalent_Prob_exp}, we can solve Problem~\ref{Prob:exp_eq} instead of Problem~\ref{Prob:fixed_step_size} with $m=E$.
As the form of Problem~\ref{Prob:exp_eq} is similar to that of Problem~\ref{Prob:constant_eq},
we propose an iterative algorithm to obtain a KKT point of Problem~\ref{Prob:exp_eq} using the methods proposed in Sec.~\ref{SubSec:Solution_Fixed_constant}.
Specifically, at iteration $t$,
update $\left(\mathbf K^{(t)},B^{(t)},T_2^{(t)},X_0^{(t)}\right)$ by solving Problem~\ref{Prob:exp_GP}, which is parameterized by $\left(\mathbf K^{(t-1)},B^{(t-1)},T_2^{(t-1)},X_0^{(t-1)}\right)$ obtained at iteration $t-1$.
\begin{algorithm}[t]
\caption{Algorithm for Obtaining a KKT Point of Problem~\ref{Prob:exp_eq}}
\label{Alg:exp}
\begin{algorithmic}[1]
{\small
\STATE {\bf{Initialize:}}
Choose any feasible solution $\left(\mathbf K^{(0)}, B^{(0)}, T_1^{(0)}, T_2^{(0)},X_0^{(0)}\right)$ of Problem~\ref{Prob:exp_eq},
and set $t=1$.
\REPEAT
    \STATE
    Compute $\left(\mathbf K^{(t)}, B^{(t)}, T_1^{(t)}, T_2^{(t)},X_0^{(t)}\right)$ by transforming Problem~\ref{Prob:exp_GP} into a GP in convex form
    and solving it with standard convex optimization techniques.
    \STATE Set $t:=t+1$.
\UNTIL{Some convergence criteria is met.}
}
\end{algorithmic}
\end{algorithm}

\begin{Prob}[Approximate GP of Problem~\ref{Prob:exp_eq} at Iteration $t$]\label{Prob:exp_GP}
For any given $\gamma_E\in\left(0,\frac{1}{L}\right]$ and $\rho_E\in(0,1)$,
\begin{align}
&\min_{\substack{\mathbf K\succ\mathbf0,B,T_1,T_2,X_0>0}}
{\quad}E(\mathbf K,B)\nonumber\\
&\mathrm{s.t.}{\quad}
\eqref{eq:fix_epi_cons1},\ \eqref{eq:fix_epi_cons2},\ \eqref{eq:cons_X_0},\nonumber\\
&{\quad}\frac{a_1c_1+\left(a_2c_2T_2^2+\frac{a_3c_3}{B}+C_{\max}X_0\right)\sum_{n\in\mathcal N} K_n+a_3c_4\sum_{n\in\mathcal N}q_{s_0,s_n}K_n^2}
{\prod\limits_{n\in\mathcal N}\left(\frac{C_{\max}K_n}{\lambda_{1,n}^{(t-1)}}\right)^{\lambda_{1,n}^{(t-1)}}
\left(\frac{a_2c_2T_2^2X_0^3K_n}{\lambda_{2,n}^{(t-1)}}\right)^{\lambda_{2,n}^{(t-1)}}
\left(\frac{a_3c_3X_0^2K_n}{B\lambda_{3,n}^{(t-1)}}\right)^{\lambda_{3,n}^{(t-1)}}
\left(\frac{a_3c_4q_{s_0,s_n}X_0^2K_n^2}{\lambda_{4,n}^{(t-1)}}\right)^{\lambda_{4,n}^{(t-1)}}
}\leq1,\label{eq:exp_cons_C_approx}\\
&{\quad}\frac{\ln\frac{1}{X_0^{(t-1)}}X_0+X_0^{(t-1)}}
{X_0\left(\frac{\left(K_0^{(t-1)}\ln\frac{1}{\rho_E}+1\right)K_0}{K_0^{(t-1)}}\right)
^{\frac{K_0^{(t-1)}\ln\frac{1}{\rho_E}}{K_0^{(t-1)}\ln\frac{1}{\rho_E}+1}}
\left(K_0^{(t-1)}\ln\frac{1}{\rho_E}+1\right)
^{\frac{1}{{K_0^{(t-1)}\ln\frac{1}{\rho_E}+1}}}}\leq1,\label{eq:exp_cons_X0_1_approx}\\
&{\quad}\frac{\frac{X_0}{X_0^{(t-1)}}+K_0\ln\frac{1}{\rho_E}}{\ln\frac{1}{X_0^{(t-1)}}+1}\leq1,\label{eq:exp_cons_X0_2_approx}
\end{align}
where $\beta_n^{(t-1)}\triangleq\frac{K_n^{(t-1)}}{\sum_{n\in\mathcal N}K_n^{(t-1)}}$,
$\lambda_0^{(t-1)}\triangleq\left(C_{\max}
+a_2c_2\left(T_2^{(t-1)}\right)^2\left(X_0^{(t-1)}\right)^3
+\frac{a_3c_3\left(X_0^{(t-1)}\right)^2}{B^{(t-1)}}\right)\\\sum_{n\in\mathcal N}K_n^{(t-1)}
+a_3c_4\left(X_0^{(t-1)}\right)^2\sum_{n\in\mathcal N} q_{s_0,s_n}\left(K_n^{(t-1)}\right)^2$,
$\lambda_{1,n}^{(t-1)}\triangleq
\frac{C_{\max}K_n^{(t-1)}}{\lambda_0^{(t-1)}}$,
$\lambda_{2,n}^{(t-1)}\triangleq\\
\frac{a_2c_2\left(T_2^{(t-1)}\right)^2\left(X_0^{(t-1)}\right)^3K_n^{(t-1)}}
{\lambda_0^{(t-1)}}$,
$\lambda_{3,n}^{(t-1)}\triangleq
\frac{a_3c_3\left(X_0^{(t-1)}\right)^2K_n^{(t-1)}}
{B^{(t-1)}\lambda_0^{(t-1)}}$,
$\lambda_{4,n}^{(t-1)}\triangleq
\frac{a_3c_4q_{s_0,s_n}\left(X_0^{(t-1)}\right)^2\left(K_n^{(t-1)}\right)^2}
{\lambda_0^{(t-1)}}$, and $\left(\mathbf K^{(t)}, B^{(t)}, T_1^{(t)}, T_2^{(t)},X_0^{(t)}\right)$ denotes an optimal solution of Problem~\ref{Prob:exp_GP}.
\end{Prob}
The constraint functions in \eqref{eq:exp_cons_C_approx}, \eqref{eq:exp_cons_X0_1_approx}, and \eqref{eq:exp_cons_X0_2_approx}, constructed based on the arithmetic-geometric mean inequality and first-order Taylor series expansion, are approximations of the constraint functions in \eqref{eq:exp_cons_C}, \eqref{eq:exp_cons_X0_1}, and \eqref{eq:exp_cons_X0_2} at $(\mathbf K^{(t-1)},B^{(t-1)},T_2^{(t-1)},X_0^{(t-1)})$, respectively, and are posynomials.
Consequently, Problem~\ref{Prob:exp_GP} is a standard GP and can be readily transformed into a convex problem and solved by using standard convex optimization techniques. If an interior-point method is applied, the computational complexity for solving Problem~\ref{Prob:exp_GP} is $\mathcal O(N^{3.5})$ \cite{CVX}.
The details are summarized in Algorithm~\ref{Alg:exp}.
Analogously, we have the following result.
\begin{Thm}[Convergence of Algorithm~\ref{Alg:exp}]\label{Thm:exp_convergence}
$\left(\mathbf K^{(t)},B^{(t)},T_1^{(t)},T_2^{(t)},X_0^{(t)}\right)$ obtained by Algorithm~\ref{Alg:exp} converges to a KKT point of Problem~\ref{Prob:exp_eq}, as $t\rightarrow\infty$.
\end{Thm}
\begin{IEEEproof}
By the first-order Taylor series expansion \cite{MM},
we can show that the approximations of $X_0\ln\frac{1}{X_0}$ in \eqref{eq:exp_cons_X0_1} and $\ln\frac{1}{X_0}$ in \eqref{eq:exp_cons_X0_2}, i.e.,
$(\ln\frac{1}{\hat{X_0}}-1)X_0+\hat{X_0}$
and
$-\frac{X_0}{\hat{X_0}}+\ln\frac{1}{\hat{X_0}}+1$,
satisfy Properties (i), (ii), and (iii) in \cite{GIA}.
Following the proof of Theorem~\ref{Thm:constant_convergence},
we can show Theorem~\ref{Thm:exp_convergence}.\end{IEEEproof}

\subsection{Solution to Problem~\ref{Prob:fixed_step_size} with $m=D$}\label{SubSec:Solution_Fixed_dim}
Compared to Problem~\ref{Prob:fixed_step_size} with $m=C$,
Problem~\ref{Prob:fixed_step_size} with $m=D$ has an extra challenge caused by the products of logarithmic functions and posynomials.
To address the challenge caused by the non-differentiable constraint functions in \eqref{eq:Cons_time} and \eqref{eq:Cons_conv} with $m=D$ as well as the extra challenge,
we equivalently transform Problem~\ref{Prob:fixed_step_size} with $m=D$ into the following problem.
\begin{Prob}[Equivalent Problem of Problem~\ref{Prob:fixed_step_size} with $m=D$]\label{Prob:dim_eq}
For any given $\gamma_D\in\left(0,\frac{1}{L}\right]$ and $\rho_D>0$,
\begin{align}
&\min_{\substack{\mathbf K\succ\mathbf0,B,T_1,T_2>0}}{\quad}E(\mathbf K,B)\nonumber\\
&\mathrm{s.t.}{\quad}\eqref{eq:fix_epi_cons1},\ \eqref{eq:fix_epi_cons2},\ \eqref{eq:fix_epi_cons_T},\nonumber\\
&\frac{b_1c_1K_0}{\sum_{n\in\mathcal N}K_n}
\!+\!b_2c_2T_2^2K_0
\!+\!\frac{b_3c_3K_0}{B}
\!+\!\frac{b_3c_4K_0\sum_{n\in\mathcal N}q_{s_0,s_n}K_n^2}{\sum_{n\in\mathcal N}K_n}
\leq C_{\max}K_0\ln\left(\frac{K_0+\rho_D+1}{\rho_D+1}\right)\!.\!\label{eq:dim_cons_C_eq}
\end{align}
\end{Prob}
\begin{Thm}[Equivalence between Problem~\ref{Prob:fixed_step_size} with $m=D$ and Problem~\ref{Prob:dim_eq}]\label{Thm:Equivalent_Prob_dim}
If $\left(\mathbf K^*,B^*,T_1^*,T_2^*\right)$ is an optimal point of Problem~\ref{Prob:dim_eq},
then $\left(\mathbf K^*,B^*\right)$ is an optimal point of Problem~\ref{Prob:fixed_step_size} with $m=D$.
\end{Thm}
\begin{IEEEproof}
Following the proof of Theorem~\ref{Thm:Equivalent_Prob_cons}, we readily show Theorem~\ref{Thm:Equivalent_Prob_dim}.
\end{IEEEproof}
By Theorem~\ref{Thm:Equivalent_Prob_dim}, we can solve Problem~\ref{Prob:dim_eq} instead of Problem~\ref{Prob:fixed_step_size} with $m=D$.
As the form of Problem~\ref{Prob:dim_eq} is similar to that of Problem~\ref{Prob:constant_eq},
we propose an iterative algorithm to obtain a KKT point of Problem~\ref{Prob:dim_eq} using the methods proposed in Sec.~\ref{SubSec:Solution_Fixed_constant}.
Specifically, at iteration $t$,
update $\left(\mathbf K^{(t)}, B^{(t)}\right)$ by solving Problem~\ref{Prob:dim_GP}, which is parameterized by $\mathbf K^{(t-1)}$ obtained at iteration $t-1$.
\begin{algorithm}[t]
\caption{Algorithm for Obtaining a KKT Point of Problem~\ref{Prob:dim_eq}}
\label{Alg:dim}
\begin{algorithmic}[1]
\small
\STATE {\bf{Initialize:}}
Choose any feasible solution $\left(\mathbf K^{(0)}, B^{(0)}, T_1^{(0)}, T_2^{(0)}\right)$ of Problem~\ref{Prob:dim_eq},
and set $t=1$.
\REPEAT
    \STATE
    Compute $\left(\mathbf K^{(t)}, B^{(t)}, T_1^{(t)}, T_2^{(t)}\right)$ by transforming Problem~\ref{Prob:dim_GP} into a GP in convex form
    and solving it with standard convex optimization techniques.
    \STATE Set $t:=t+1$.
\UNTIL{Some convergence criteria is met.}
\normalsize
\end{algorithmic}
\end{algorithm}

\begin{Prob}[Approximate GP of Problem~\ref{Prob:dim_eq} at Iteration $t$]\label{Prob:dim_GP}
For any given $\gamma_D\in\left(0,\frac{1}{L}\right]$ and $\rho_D>0$,
\begin{align}
&\min_{\substack{\mathbf K\succ\mathbf0,B,T_1,T_2>0}}{\quad}E(\mathbf K,B)\nonumber\\
\mathrm{s.t.}&{\quad}\eqref{eq:fix_epi_cons1},\ \eqref{eq:fix_epi_cons2},\ \eqref{eq:fix_epi_cons_T},\nonumber\\
&{\quad}\frac{\frac{b_1c_1}{\prod\limits_{n\in\mathcal N}\left(\frac{K_n}{\beta_n^{(t-1)}}\right)^{\beta_n^{(t-1)}}}+b_2c_2T_2^2+\frac{b_3c_3}{B}
+\frac{b_3c_4\sum_{n\in\mathcal N}q_{s_0,s_n}K_n^2}{\prod_{n\in\mathcal N}\left(\frac{K_n}{\beta_n^{(t-1)}}\right)^{\beta_n^{(t-1)}}}+\frac{C_{\max}\left(K_0^{(t-1)}\right)^2}{K_0\left(K_0^{(t-1)}+\rho_D+1\right)}}
{C_{\max}\left(\ln\left(\frac{K_0^{(t-1)}+\rho_D+1}{\rho_D+1}\right)
+\frac{K_0^{(t-1)}}{K_0^{(t-1)}+\rho_D+1}\right)}\leq1,\label{eq:dim_cons_C_approx}
\end{align}
where $\beta_n^{(t-1)}\triangleq\frac{K_n^{(t-1)}}{\sum_{n\in\mathcal N}K_n^{(t-1)}}$,
and $\left(\mathbf K^{(t)}, B^{(t)}, T_1^{(t)}, T_2^{(t)}\right)$ denotes an optimal solution of Problem~\ref{Prob:dim_GP}.
\end{Prob}
The constraint function in \eqref{eq:dim_cons_C_approx}, constructed based on the arithmetic-geometric mean inequality and first-order Taylor series expansion, is an approximation of the constraint function in \eqref{eq:dim_cons_C_eq} at $\mathbf K^{(t-1)}$ and is a posynomial.
Thus, Problem~\ref{Prob:dim_GP} is a standard GP and can be readily transformed into a convex problem and solved by using standard convex optimization techniques. If an interior-point method is applied, the computational complexity for solving Problem~\ref{Prob:dim_GP} is $\mathcal O(N^{3.5})$ \cite{CVX}.
The details are summarized in Algorithm~\ref{Alg:dim}.
Analogously, we have the following result.
\begin{Thm}[Convergence of Algorithm~\ref{Alg:dim}]\label{Thm:dim_convergence}
$\left(\mathbf K^{(t)},B^{(t)},T_1^{(t)},T_2^{(t)}\right)$ obtained by Algorithm~\ref{Alg:dim} converges to a KKT point of Problem~\ref{Prob:dim_eq}, as $t\rightarrow\infty$.
\end{Thm}
\begin{IEEEproof}
By the first-order Taylor series expansion \cite{MM},
we can show that the approximation of $K_0\ln\left({\frac{K_0+\rho+1}{\rho+1}}\right)$ in \eqref{eq:dim_cons_C_eq}, i.e., $\left(\ln\left(\frac{\hat{K_0}+\rho+1}{\rho+1}\right)+\frac{\hat{K_0}}{\hat{K_0}+\rho+1}\right)K_0-\frac{\hat{K_0}^2}{\hat{K_0}+\rho+1}$,
satisfies Properties (i), (ii), and (iii) in \cite{GIA}.
Following the proof of Theorem~\ref{Thm:constant_convergence},
we can show Theorem~\ref{Thm:dim_convergence}.
\end{IEEEproof}

\section{Optimization of All Algorithm Parameters}\label{sec:Opt_adj_step_size}
The step size sequence of an FL algorithm chosen only for improving the convergence performance of centralized learning may not yield satisfactory communication time and communication energy when implementing it in an edge computing system.
To explore the full potential of FL, in this section, we optimize all the algorithm parameters, including the numbers of global and local iterations, mini-batch size, and step size sequence, to minimize the energy cost under the time constraint and convergence error constraint,
as illustrated in Fig.~\ref{Fig:Structure_2}.
\begin{figure}[t]
\setlength{\belowcaptionskip}{-10pt}
\begin{center}
{\includegraphics[width=300pt]{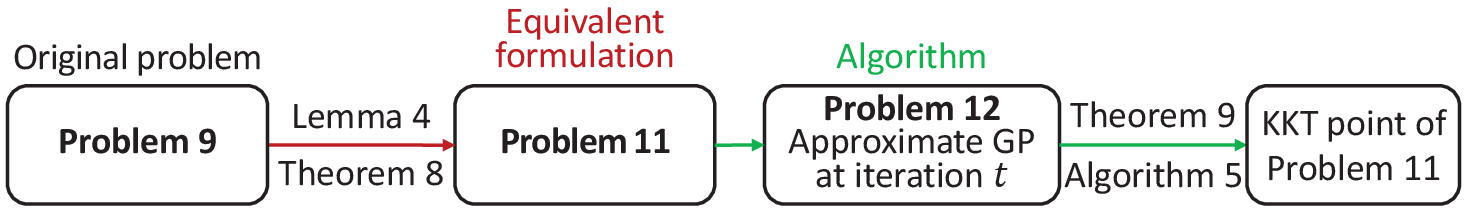}}
\caption{\small{Proposed solution to optimization of all algorithm parameters.}}\label{Fig:Structure_2}
\end{center}
\end{figure}

\subsection{Problem Formulation}\label{SubSec:Formulation_Adj_quantization}
We optimize $\mathbf K$, $B$, and $\mathbf\Gamma$ to minimize $E(\mathbf K,B)$ under the constraints in \eqref{eq:Cons_gamma}, \eqref{eq:Cons_time}, and \eqref{eq:Cons_conv} with $m=A$.
\begin{Prob}[Optimization of All Algorithm Parameters]\label{Prob:adj_step_size}
\begin{align}
\min_{\substack{\mathbf K,\mathbf\Gamma\succ\mathbf 0,B>0}}&{\quad}E(\mathbf K,B)\nonumber\\[-8pt]
\mathrm{s.t.}&{\quad}\eqref{eq:Cons_gamma},\ \eqref{eq:Cons_time},\nonumber\\[-5pt]
&{\quad}C_A(\mathbf K,B,\mathbf\Gamma)\leq C_{\max}\label{eq:Cons_conv_arb}.
\end{align}
\end{Prob}

Similar to Problem~\ref{Prob:fixed_step_size}, Problem~\ref{Prob:adj_step_size} is a non-convex problem with non-differentiable constraint functions.
Different from most optimization problems, the dimension of $\mathbf\Gamma$, i.e., $K_0$, is a variable, making Problem~\ref{Prob:adj_step_size} very challenging.

\subsection{Solution}\label{SubSec:Solution_Adj_step_size}
In this part, we develop an iterative algorithm to obtain a KKT point of an equivalent problem of Problem~\ref{Prob:adj_step_size}.
To address the challenge due to the dimension-varying optimization variable $\mathbf\Gamma$, we first consider the minimization of the convergence error with respect to the step size sequence $\mathbf\Gamma$ satisfying \eqref{eq:Cons_gamma} with fixed length $K_0$ subject to the following constant:
\begin{align}\label{eq:sum_constant}
\sum_{k_0\in\mathcal K_0}\gamma^{(k_0)}=S,
\end{align}
where $S\in\left(0,\frac{K_0}{L}\right]$ represents the sum of the elements in $\mathbf\Gamma$.\footnote{The range of $S$, i.e., $\left(0,\frac{K_0}{L}\right]$, is determined by \eqref{eq:Cons_gamma}.}
\begin{Prob}[Convergence Error Minimization]\label{Prob:step_size_sequence}
For any given $\mathbf K\in\mathbb Z_+^{N+1}$, $B\in\mathbb Z_+$, and $S\in\left(0,\frac{K_0}{L}\right]$,
\begin{align}
\min_{\mathbf\Gamma}&{\quad}C_A(\mathbf K,B,\mathbf\Gamma)\nonumber\\
\mathrm{s.t.}&{\quad}\eqref{eq:Cons_gamma},\ \eqref{eq:sum_constant}.\nonumber
\end{align}
\end{Prob}

Then, we solve Problem~\ref{Prob:step_size_sequence}.
\begin{Lem}[Optimal Step Size Sequence $\mathbf\Gamma$]\label{Lem:optimal_step_size}
An optimal solution of Problem~\ref{Prob:step_size_sequence} is $\frac{S}{K_0}\mathbf1$.
\end{Lem}
\begin{IEEEproof}
By Cauchy-Schwartz inequality, we have
$\frac{\sum_{k_0\in\mathcal K_0}\left(\gamma^{(k_0)}\right)^3}{\sum_{k_0\in\mathcal K_0}\gamma^{(k_0)}}
=\frac{\sum_{k_0\in\mathcal K_0}\left(\gamma^{(k_0)}\right)^3}{S}
{\overset{(a)}{\geq}}\\\frac{\left(\sum_{k_0\in\mathcal K_0}\gamma^{(k_0)}\right)^3}{SK_0}=\frac{S^2}{K_0}$
and
$\frac{\sum_{k_0\in\mathcal K_0}\left(\gamma^{(k_0)}\right)^2}{\sum_{k_0\in\mathcal K_0}\gamma^{(k_0)}}
=\frac{\sum_{k_0\in\mathcal K_0}\left(\gamma^{(k_0)}\right)^2}{S}
{\overset{(b)}{\geq}}\frac{\left(\sum_{k_0\in\mathcal K_0}\gamma^{(k_0)}\right)^2}{SK_0}
=\frac{S}{K_0}$,
where (a) and (b) take the equality if and only if $\gamma^{(k_0)}=\frac{S}{K_0}$ for all $k_0\in\mathcal K_0$.
Note that
the coefficients of the terms $\frac{\sum_{k_0\in\mathcal K_0}\left(\gamma^{(k_0)}\right)^3}{\sum_{k_0\in\mathcal K_0}\gamma^{(k_0)}}$ and $\frac{\sum_{k_0\in\mathcal K_0}\left(\gamma^{(k_0)}\right)^2}{\sum_{k_0\in\mathcal K_0}\gamma^{(k_0)}}$ are positive,
and the term $\frac{1}{\sum_{k_0\in\mathcal K_0}\gamma^{(k_0)}}=\frac{1}{S}$ is a constant.
Thus, for any given $\mathbf K\in\mathbb Z_+^{N+1}$, $B\in\mathbb Z_+$, and $S\in\left(0,\frac{K_0}{L}\right]$, $C_A(\mathbf K,B,\mathbf\Gamma)$ is minimized at $\frac{S}{K_0}\mathbf1$, which satisfies all constraints of Problem~\ref{Prob:step_size_sequence}.
Therefore, we can show Lemma~\ref{Lem:optimal_step_size}.
\end{IEEEproof}
Lemma~\ref{Lem:optimal_step_size} indicates that the constant step size rule achieves the minimum convergence error with finite iterations.
Based on Lemma~\ref{Lem:optimal_step_size}, we replace $\gamma^{(k_0)}$ in \eqref{eq:Cons_gamma} and \eqref{eq:Cons_conv_arb} with $\gamma$ for all $k_0\in\mathcal K_0$ and optimize $\gamma$ instead.
Besides, we adopt the method proposed in Sec.~\ref{SubSec:Solution_Fixed_constant} to address the challenge caused by the non-differentiable constraint functions in \eqref{eq:Cons_time} and \eqref{eq:Cons_conv_arb}.
Consequently, we can equivalently convert Problem~\ref{Prob:adj_step_size} into the following problem.
\begin{Prob}[Equivalent Problem of Problem~\ref{Prob:adj_step_size}]\label{Prob:adj_step_size_epi}
\begin{align}
\min_{\substack{\mathbf K\succ\mathbf0,B,\gamma,T_1,T_2>0}}&{\quad}E(\mathbf K,B)\nonumber\\
\mathrm{s.t.}&{\quad}\eqref{eq:fix_epi_cons1},\ \eqref{eq:fix_epi_cons2},\ \eqref{eq:fix_epi_cons_T},\nonumber\\
&{\quad}\frac{c_1}{\gamma K_0\sum_{n\in\mathcal N}K_n}+c_2\gamma^2T_2^2
+\frac{c_3\gamma}{B}
+\frac{c_4\gamma\sum_{n\in\mathcal N}q_{s_0,s_n}K_n^2}{\sum_{n\in\mathcal N}K_n}
\leq C_{\max},\label{eq:adj_step_size_epi_cons_C}\\
&{\quad}0\leq\gamma\leq1/L.\label{eq:adj_step_size_epi_cons_gamma}
\end{align}
\end{Prob}

\begin{Thm}[Equivalence between Problem~\ref{Prob:adj_step_size} and Problem~\ref{Prob:adj_step_size_epi}]\label{Thm:Equivalent_Prob_adj}
If $\left(\mathbf K^*, B^*, \gamma^*, T_1^*, T_2^*\right)$ is an optimal point of Problem~\ref{Prob:adj_step_size_epi},
then $\left(\mathbf K^*, B^*, \gamma^*\mathbf1\right)$ is an optimal point of Problem~\ref{Prob:adj_step_size}.
\end{Thm}
\begin{IEEEproof}
See Appendix E.
\end{IEEEproof}
By Theorem~\ref{Thm:Equivalent_Prob_adj}, we can solve Problem~\ref{Prob:adj_step_size_epi} instead of Problem~\ref{Prob:adj_step_size}.
As the form of Problem~\ref{Prob:adj_step_size_epi} is similar to that of Problem~\ref{Prob:constant_eq},
we propose an iterative algorithm to obtain a KKT point of Problem~\ref{Prob:adj_step_size_epi} using the methods proposed in Sec.~\ref{SubSec:Solution_Fixed_constant}.
Specifically, at iteration $t$,
update $\left(\mathbf K^{(t)},\right.$ $\left.B^{(t)},\gamma^{(t)}\right)$ by solving Problem~\ref{Prob:adj_step_size_GP}, which is parameterized by $\mathbf K^{(t-1)}$ obtained at iteration $t-1$.
\begin{algorithm}[t]
\caption{Algorithm for Obtaining a KKT Point of Problem~\ref{Prob:adj_step_size_epi}}
\label{Alg:adj_step_size}
\begin{algorithmic}[1]
\begin{small}
\STATE {\bf{Initialize:}} Choose any feasible solution $\left(\mathbf K^{(0)},\gamma^{(0)},B^{(0)},T_1^{(0)},T_2^{(0)}\right)$ of Problem~\ref{Prob:adj_step_size_epi}, and set $t=1$.
\REPEAT
    \STATE Compute $\left(\mathbf K^{(t)},\gamma^{(t)},B^{(t)},T_1^{(t)}, T_2^{(t)}\right)$ by transforming Problem~\ref{Prob:adj_step_size_GP} into a GP in convex form and solving it with standard convex optimization techniques.
    \STATE Set $t:=t+1$.
\UNTIL{Some convergence criteria is met.}
\end{small}
\end{algorithmic}
\end{algorithm}

\begin{Prob}[Approximate GP of Problem~\ref{Prob:adj_step_size_epi} at Iteration $t$]\label{Prob:adj_step_size_GP}
\begin{align}
&\min_{\substack{\mathbf K\succ\mathbf0,B,\gamma,T_1,T_2>0}}{\quad}E(\mathbf K,B)\nonumber\\
&\mathrm{s.t.}{\quad}\eqref{eq:fix_epi_cons1},\ \eqref{eq:fix_epi_cons2},\
\eqref{eq:fix_epi_cons_T},\ \eqref{eq:adj_step_size_epi_cons_gamma},\nonumber\\
&{\quad}\frac{c_1}{\gamma K_0\prod_{n\in\mathcal N}\left(\frac{K_n}{\beta_n^{(t-1)}}\right)^{\beta_n^{(t-1)}}}+c_2\gamma^2T_2^2
+\frac{c_3\gamma}{B}
+\frac{c_4\gamma\sum_{n\in\mathcal N}q_{s_0,s_n}K_n^2}{\prod_{n\in\mathcal N}\left(\frac{K_n}{\beta_n^{(t-1)}}\right)^{\beta_n^{(t-1)}}}\leq C_{\max},\label{eq:adj_step_size_epi_cons_C_approx}
\end{align}
where $\beta_n^{(t-1)}\triangleq\frac{K_n^{(t-1)}}{\sum_{n\in\mathcal N}K_n^{(t-1)}}$
and $\left(\mathbf K^{(t)},B^{(t)},\gamma^{(t)}\!,T_1^{(t)},T_2^{(t)}\right)\!$ denotes an optimal solution of Problem~\ref{Prob:adj_step_size_GP}.
\end{Prob}
The constraint function in \eqref{eq:adj_step_size_epi_cons_C_approx},
constructed based on the arithmetic-geometric mean inequality,
is an approximation of the constraint function in \eqref{eq:adj_step_size_epi_cons_C} at $\mathbf K^{(t-1)}$ and is a posynomial.
Consequently, Problem~\ref{Prob:adj_step_size_GP} is a standard GP and can be readily transformed into a convex problem and solved by using standard convex optimization techniques.
If an interior-point method is applied, the computational complexity for solving Problem~\ref{Prob:adj_step_size_GP} is $\mathcal O(N^{3.5})$ \cite{CVX}.
The details are summarized in Algorithm~\ref{Alg:adj_step_size}.
Analogously, we have the following result.

\begin{Thm}[Convergence of Algorithm~\ref{Alg:adj_step_size}]\label{Thm:adj_step_size_convergence}
$\left(\mathbf K^{(t)},B^{(t)},\gamma^{(t)},T_1^{(t)},T_2^{(t)}\right)$ obtained by Algorithm~\ref{Alg:adj_step_size} converges to a KKT point of Problem~\ref{Prob:adj_step_size_epi}, as $t\rightarrow\infty$.
\end{Thm}
\begin{IEEEproof}
We can show Theorem~\ref{Thm:adj_step_size_convergence} following the proof of Theorem~\ref{Thm:constant_convergence} directly.
\end{IEEEproof}

\section{Numerical Results}

\begin{figure}[t!]
\centering
\begin{minipage}[c]{305pt}
	\begin{center}
		\subfigure[\scriptsize{Training loss.}\label{fig:Cost}] {\resizebox{150pt}{!}{\includegraphics{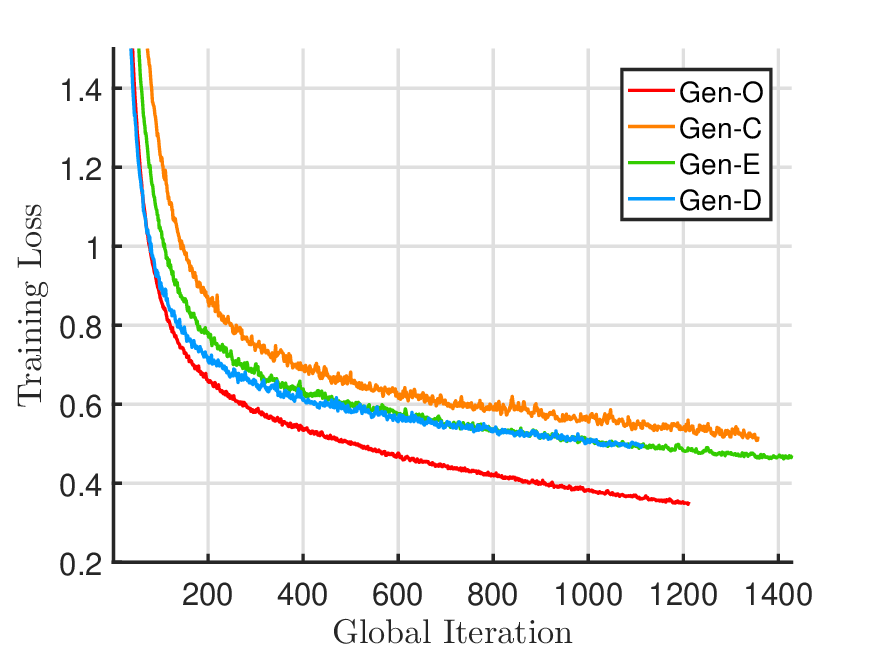}}}
		\subfigure[\scriptsize{Test accuracy.}\label{fig:Acc}]
{\resizebox{150pt}{!}{\includegraphics{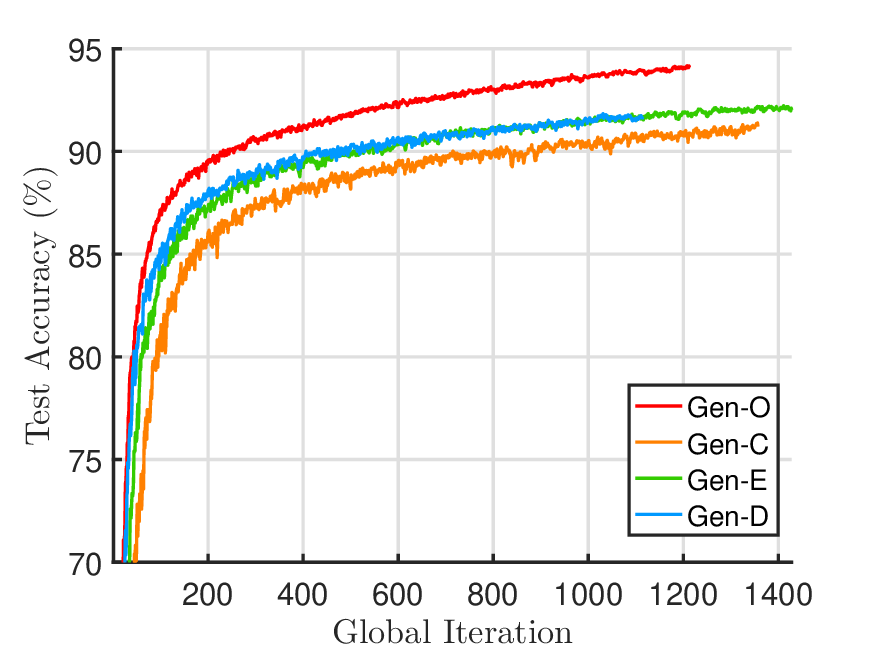}}}
	\end{center}\vspace{-20pt}
	\caption{\small{The training loss and test accuracy of GenQSGD versus global iterations at $C_{\max}=0.25$ and $T_{\max}=100000$.}}
	\label{fig:Convergence}
\end{minipage}
\begin{minipage}[c]{150pt}
\begin{center}
{\resizebox{150pt}{!}{\includegraphics{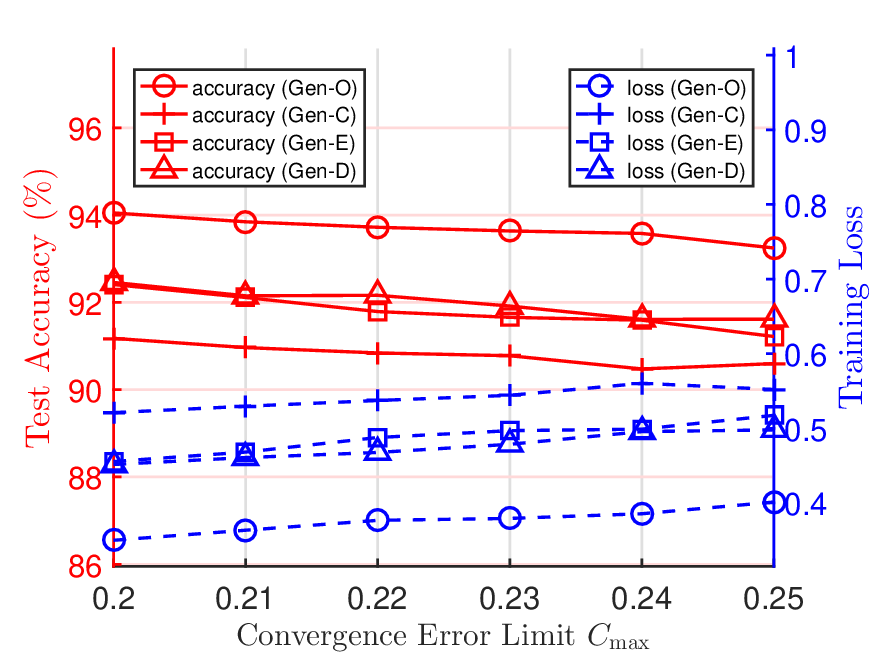}}}
\end{center}\vspace{-15pt}
\caption{\small{The training loss and test accuracy of GenQSGD versus $C_{\max}$ at $T_{\max}=100000$.}}
\label{Fig:Validity}
\end{minipage}\hspace{5pt}
\vspace{-5pt}
\begin{center}
\subfigure[\small{Energy cost versus $C_{\max}$ at $T_{\max}=10000$.}]
{\resizebox{154pt}{!}{\includegraphics{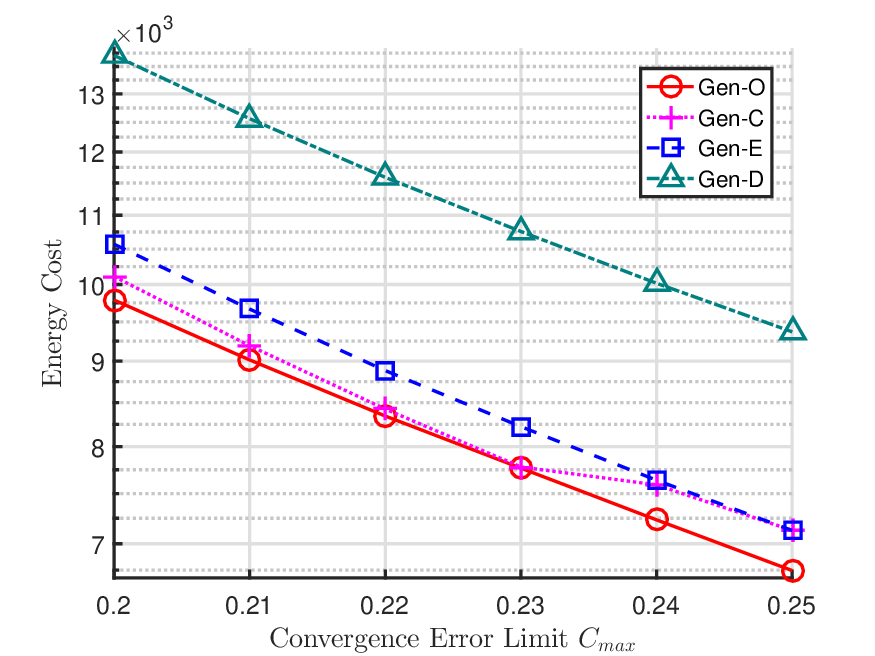}}\label{Fig:Energy_Conv}}\quad
\subfigure[\small{Energy cost versus $T_{\max}$ at $C_{\max}=0.2$.}]
{\resizebox{154pt}{!}{\includegraphics{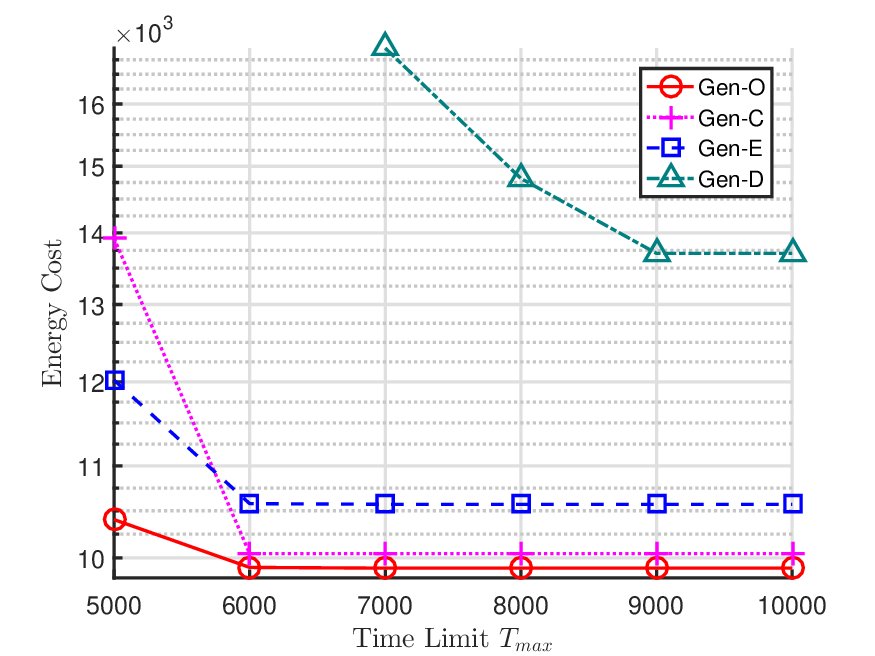}}\label{Fig:Energy_Time}}
\end{center}
\vspace{-20pt}
\caption{\small{The energy cost for implementing GenQSGD with the constant, exponential, and diminishing step size rules versus $C_{\max}$ and $T_{\max}$.}}
\label{Fig:Constraints}
\vspace{-10pt}
\end{figure}

In this section, we evaluate the performances of the proposed optimization-based GenQSGD for given and optimized step size sequences.
We consider a ten-class classification problem with the MNIST dataset ($I=6\times10^4$) and partition it into $N=10$ subsets, each of which is stored on one worker.
We consider a neural network with three layers (i.e., two-layer neural network), including an input layer composed of $784$ cells, a hidden layer composed of $128$ cells, and an output layer composed of $10$ cells.
We use the sigmoid activation function for the hidden layer and the softmax activation function for the output layer.
We also use the cross entropy loss function to measure the classification performance of the considered neural network.
For the ML problem, we set $L=0.084$, $\sigma=33.18$, and $G=33.63$, which are obtained by pre-training.
We consider the three step size sequences given by \eqref{eq:constant_rule}, \eqref{eq:exp_rule}, and \eqref{eq:dim_rule_sp} under the constant (C), exponential (E), and diminishing (D) step size rules introduced in Sec.~\ref{SubSec:Conv}, with the sequence parameters $\gamma_C=0.01$, $\gamma_E=0.02$, $\gamma_D=0.02$, $\rho_E=0.9995$ and $\rho_D=600$.
For each step size sequence, we consider PM-SGD with $K_n=1,n\in\mathcal N$ \cite{PMSGD} (PM),
FedAvg with $K_n=l\frac{I_n}{B},l\in\mathbb Z_+,n\in\mathcal N$ \cite{FedAvg} (FA),
and PR-SGD with $B=1$ \cite{YuHao} (PR), with the remaining algorithm parameters being optimized using the proposed method (opt) and fixed (fix), respectively.
The corresponding baseline FL algorithms are denoted by PM-C/E/D-opt/fix, FA-C/E/D-opt-fix, and PR-C/E/D-opt-fix, respectively.
The proposed optimization-based GenQSGD under the given three step size sequences are denoted by Gen-C/E/D, and the proposed optimization-based GenQSGD with optimized step size sequence is denoted by Gen-O.
We adopt the convergence criterion $\|\mathbf x^{(t)}-\mathbf x^{(t-1)}\|_2\leq0.01$ for Algorithms~\ref{Alg:constant}--\ref{Alg:adj_step_size}, where $\{\mathbf x^{(t)}\}$ represents the generated sequence.
For simplicity, we divide the set of workers, i.e., $\mathcal N$, into two classes, $\mathcal N_1=\{1,2,3,4,5\}$ and $\mathcal N_2=\{6,7,8,9,10\}$.
We set $F_n=F^{(1)},s_n=s^{(1)}\in\mathbb Z_+$ for all $n\in\mathcal N_1$ and $F_n=F^{(2)},s_n=s^{(2)}\in\mathbb Z_+$ for all $n\in\mathcal N_2$
such that $\frac{F^{(1)}+F^{(2)}}{2}=1\times10^9 \text{(cycles/s)}$ and $\frac{s^{(1)}+s^{(2)}}{2}=2^{14}$, respectively.
We use $\frac{F^{(1)}}{F^{(2)}}$ and $\frac{s^{(1)}}{s^{(2)}}$ to characterize the differences in computation capabilities and quantization parameters between the two classes of workers.
We choose $\frac{F^{(1)}}{F^{(2)}}=10$ and $\frac{s^{(1)}}{s^{(2)}}=1$, if not specified otherwise.
Besides, we choose
$\alpha_n\!\!=\!2\times10^{-28},n\!\!\in\!\!\setNbar$,
$F_0\!=\!3\!\times\!10^9$ (cycles/s),
$C_0\!\!=\!\!100$ (cycles),
$p_0\!\!=\!\!20$ (W),
$r_0\!=\!7.5\!\times\!10^7$ (b/s),
$C_n\!=\!1\!\times\!10^8$ (cycles), $n\!\in\!\mathcal N$,
$p_n\!=\!1.5$ (W), $n\!\in\!\mathcal N$,
and $r_n\!=\!5\!\times\!10^6$ (b/s), $n\!\in\!\mathcal N$.

\begin{figure}[t]
\begin{center}
\subfigure[\small{The constant step size rule.}]
{\resizebox{154pt}{!}{\includegraphics{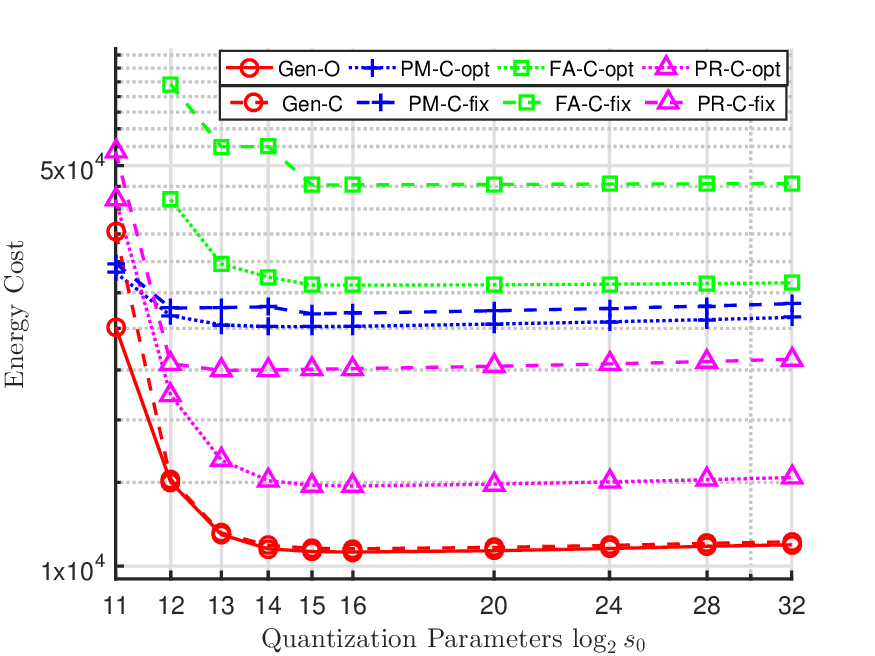}}\label{Fig:QS_cons}}
\subfigure[\small{The exponential step size rule.}]
{\resizebox{154pt}{!}{\includegraphics{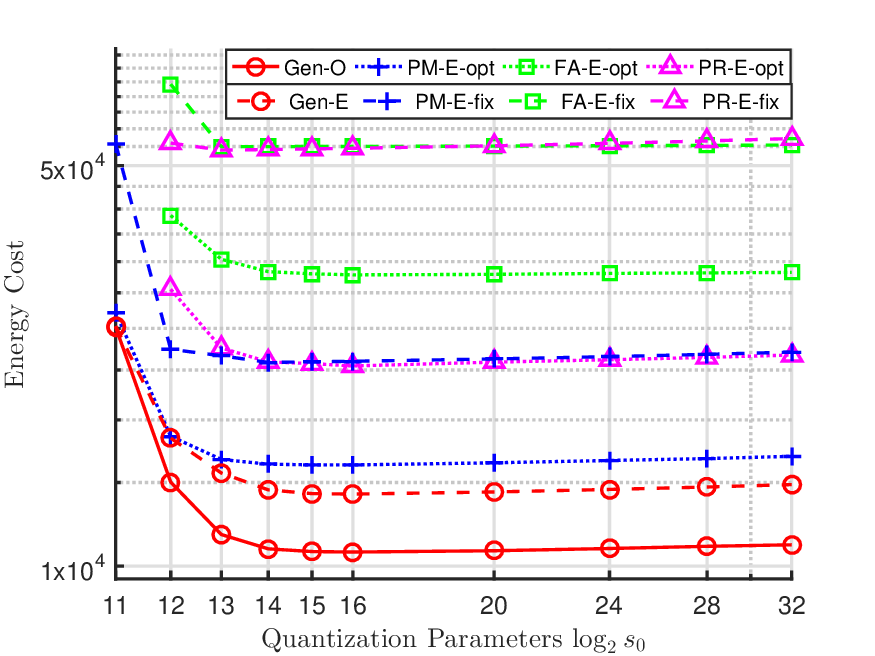}}\label{Fig:QS_exp}}
\subfigure[\small{The diminishing step size rule.}]
{\resizebox{154pt}{!}{\includegraphics{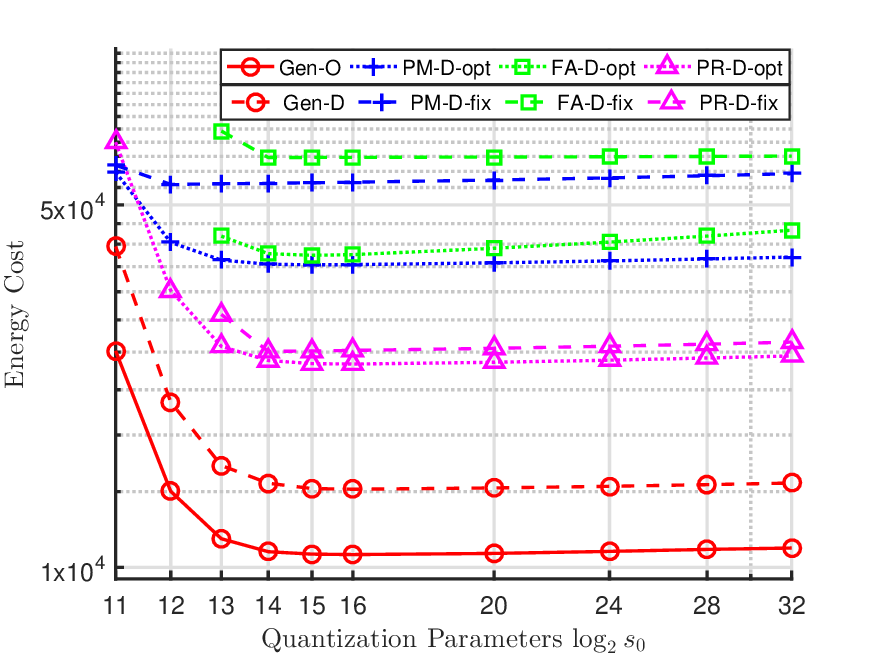}}\label{Fig:QS_dim}}
\end{center}
\vspace{-20pt}
\caption{\small{The energy cost under the constant, exponential, and diminishing step size rules versus $\log_2{s_0}$ at $C_{\max}=0.25$ and $T_{\max}=100000$.}}
\label{Fig:QS}
\vspace{-10pt}
\begin{center}
\subfigure[\small{The constant step size rule.}]
{\resizebox{154pt}{!}{\includegraphics{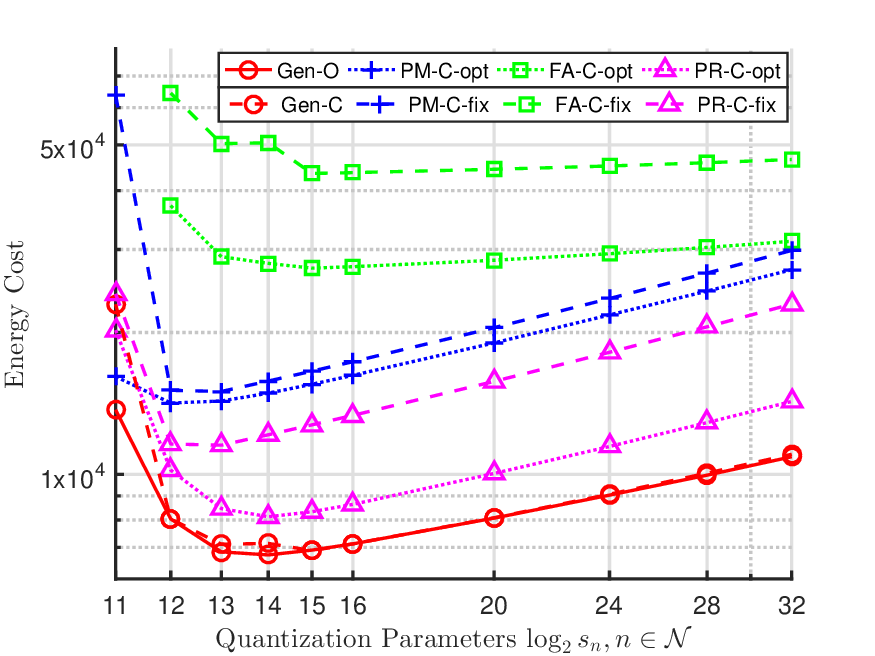}}\label{Fig:QW_cons}}
\subfigure[\small{The exponential step size rule.}]
{\resizebox{154pt}{!}{\includegraphics{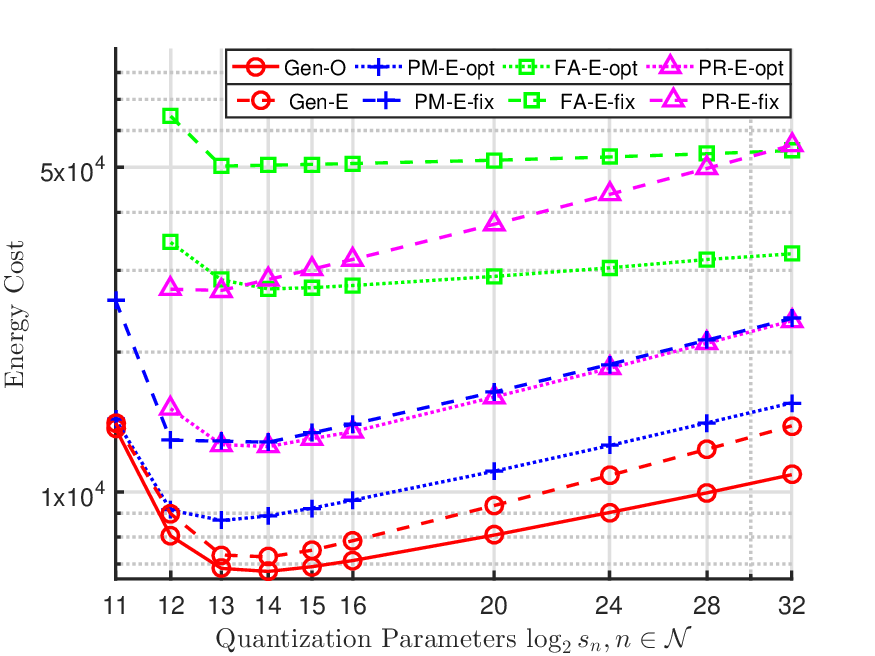}}\label{Fig:QW_exp}}
\subfigure[\small{The diminishing step size rule.}]
{\resizebox{154pt}{!}{\includegraphics{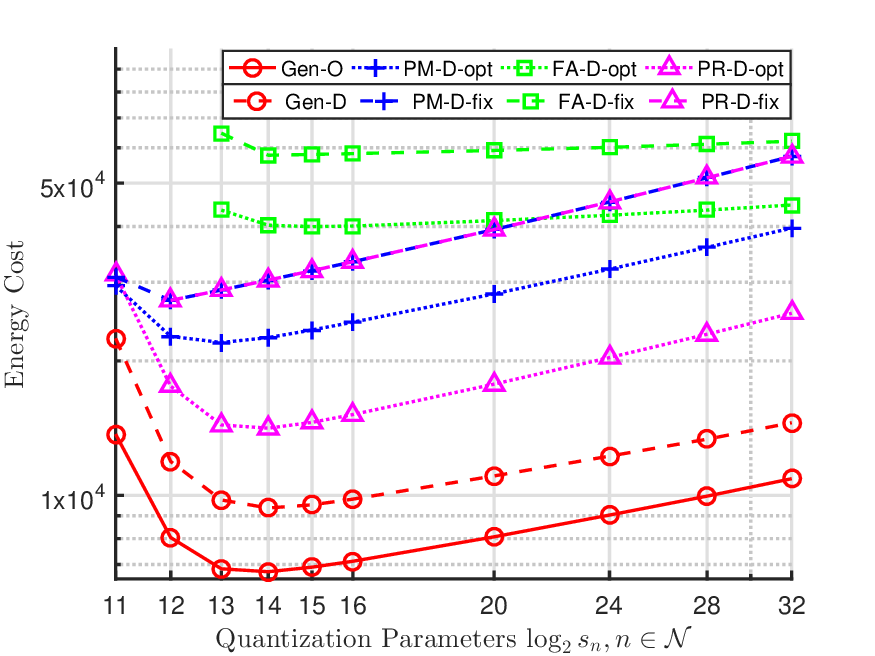}}\label{Fig:QW_dim}}
\end{center}
\vspace{-20pt}
\caption{\small{The energy cost under the constant, exponential, and diminishing step size rules versus $\log_2{s_n},n\in\mathcal N$ at $C_{\max}=0.25$ and $T_{\max}=100000$.}}
\label{Fig:QW}
\vspace{-10pt}
\end{figure}

Fig.~\ref{fig:Convergence} plots the training loss and test accuracy of the proposed optimization-based GenQSGD versus the global iteration index, which shows the convergence performance.
From Fig.~\ref{fig:Convergence}, we can see that Gen-O converges faster than Gen-C, Gen-E, and Gen-D, indicating the additional advantage of optimizing the step size sequence.
Fig.~\ref{Fig:Validity} shows the training loss and test accuracy of the proposed optimization-based GenQSGD versus the convergence error limit $C_{\max}$.
Fig.~\ref{Fig:Validity} indicates that by imposing the convergence error constraint in \eqref{eq:Cons_conv_arb}, we can effectively control the training loss and test accuracy.
Fig.~\ref{Fig:Energy_Conv} and Fig.~\ref{Fig:Energy_Time} illustrate the energy cost of the proposed optimization-based GenQSGD versus $C_{\max}$ and $T_{\max}$, respectively.
From Fig.~\ref{Fig:Constraints}, we can see that Gen-O outperforms Gen-C, Gen-E, and Gen-D, indicating the importance of optimally choosing the step size sequence for implementing GenQSGD.
Besides, Fig.~\ref{Fig:Constraints} indicates that the proposed optimization framework can achieve a trade-off among the time cost, energy cost, and convergence error.

\begin{figure}[t]
\begin{center}
\subfigure[\small{The constant step size rule.}]
{\resizebox{154pt}{!}{\includegraphics{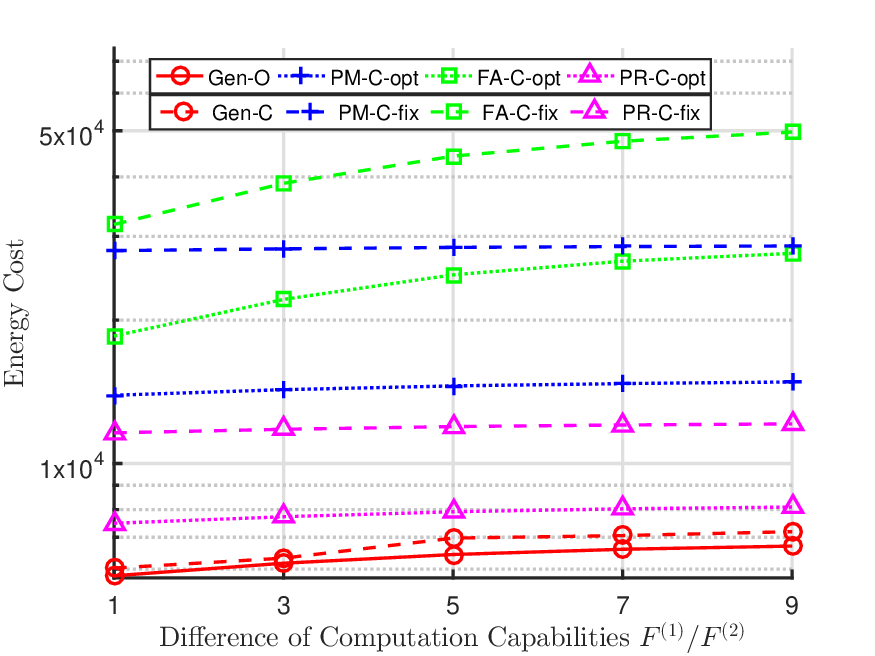}}\label{Fig:F_Hetero_cons}}
\subfigure[\small{The exponential step size rule.}]
{\resizebox{154pt}{!}{\includegraphics{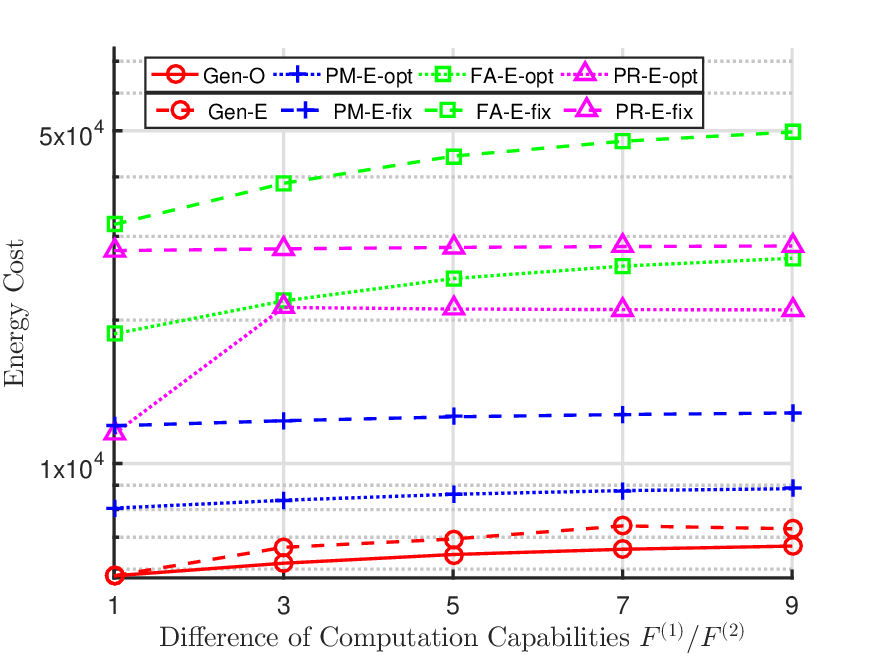}}\label{Fig:F_Hetero_exp}}
\subfigure[\small{The diminishing step size rule.}]
{\resizebox{154pt}{!}{\includegraphics{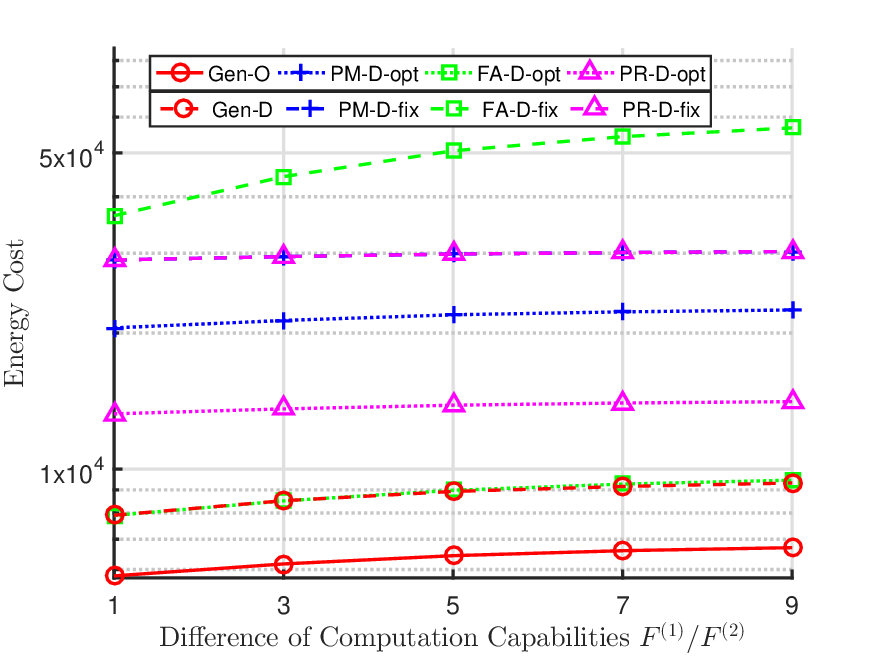}}\label{Fig:F_Hetero_dim}}
\end{center}
\vspace{-20pt}
\caption{\small{The energy cost under the constant, exponential, and diminishing step size rules versus $F^{(1)}/F^{(2)}$ at $C_{\max}=0.25$ and $T_{\max}=100000$.}}
\label{Fig:F_Hetero}
\vspace{-10pt}
\begin{center}
\subfigure[\small{The constant step size rule.}]
{\resizebox{154pt}{!}{\includegraphics{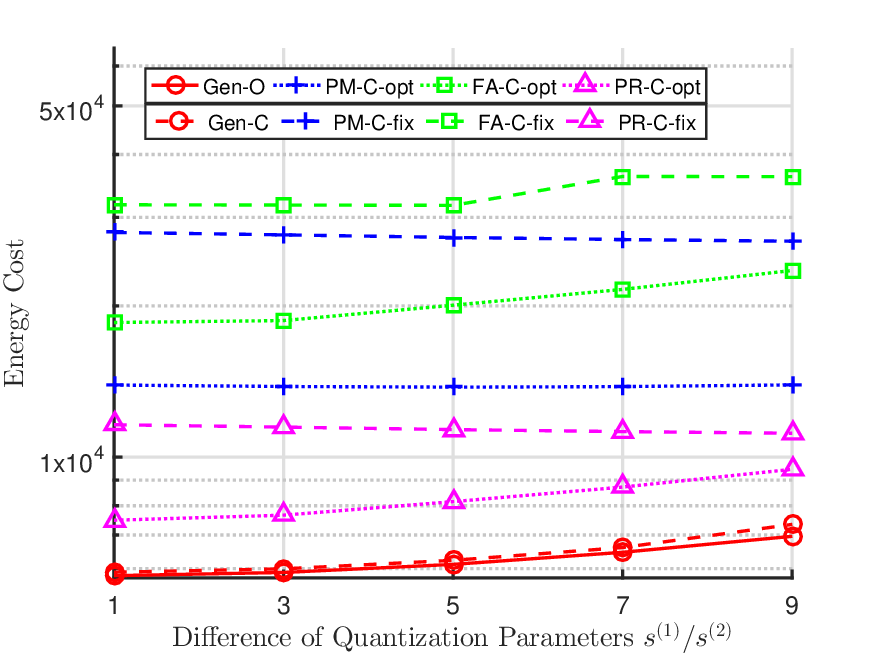}}\label{Fig:s_Hetero_cons}}
\subfigure[\small{The exponential step size rule.}]
{\resizebox{154pt}{!}{\includegraphics{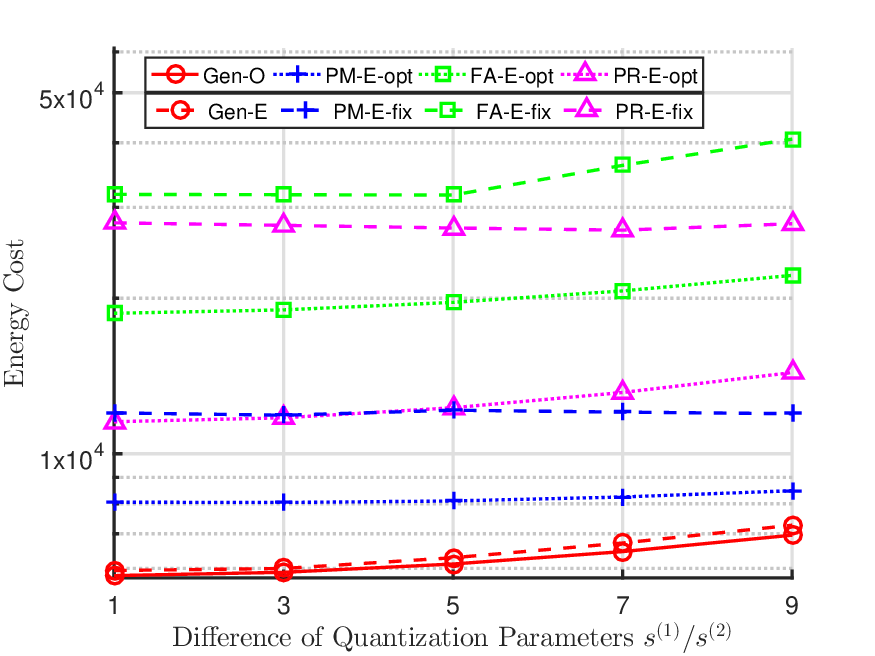}}\label{Fig:s_Hetero_exp}}
\subfigure[\small{The diminishing step size rule.}]
{\resizebox{154pt}{!}{\includegraphics{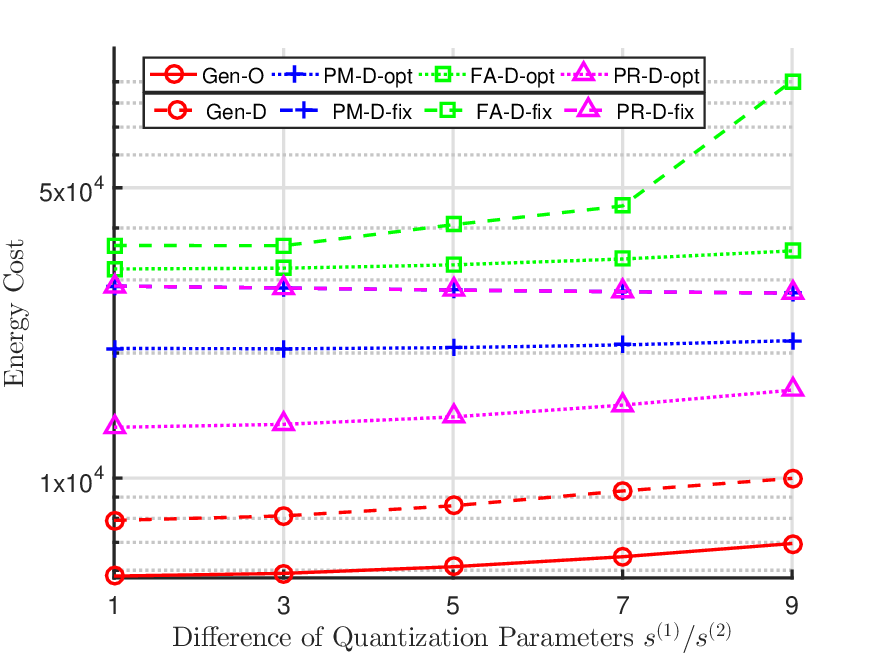}}\label{Fig:s_Hetero_dim}}
\end{center}
\vspace{-20pt}
\caption{\small{The energy cost under the constant, exponential, and diminishing step size rules versus $s^{(1)}/s^{(2)}$ at $C_{\max}=0.25$ and $T_{\max}=100000$.}}
\label{Fig:s_Hetero}
\end{figure}
\setlength{\textfloatsep}{15pt}

Fig.~\ref{Fig:QS} and Fig.~\ref{Fig:QW} illustrate the energy costs of all FL algorithms versus $\log_2s_0$ and $\log_2s_n,n\in\mathcal N$, respectively.
From Fig.~\ref{Fig:QS} and Fig.~\ref{Fig:QW}, we see that the energy cost decreases with the quantization accuracy when the quantization accuracy is small, since increasing quantization accuracy can reduce the number of global iterations to achieve a specific convergence error; and the energy cost increases with the quantization accuracy when the quantization accuracy is large, since a higher quantization accuracy leads to a larger number of transmitted bits in each global iteration.
Fig.~\ref{Fig:F_Hetero} and Fig.~\ref{Fig:s_Hetero} illustrate the energy costs of all FL algorithms versus $F^{(1)}/F^{(2)}$ and $s^{(1)}/s^{(2)}$, respectively.
From Fig.~\ref{Fig:F_Hetero} and Fig.~\ref{Fig:s_Hetero}, we see that the energy costs increase with the differences in computation capabilities and quantization parameters between the two classes of workers.
Finally, from Fig.~\ref{Fig:QS}--Fig.~\ref{Fig:s_Hetero}, we can see that for each given step size sequence, the proposed optimization-based GenQSGD significantly outperforms all the baseline FL algorithms as it allows more algorithm parameters to be optimized;
and Gen-O exhibits extra gain over Gen-C, Gen-E, and Gen-D, which is derived by the proposed optimization framework.
These observations indicate the importances of designing a general FL algorithm and optimally adapting its algorithm parameters to the underlying edge computing system and ML problem.

\section{Conclusion}
This paper studied the optimal design of FL algorithms for solving general ML problems in practical edge computing systems with quantized message passing.
We presented a general FL algorithm, namely GenQSGD, parameterized by the numbers of global and local iterations, mini-batch size, and step size sequence, and analyzed its convergence for a general (not necessarily convex) ML problem.
Moreover, we proposed a powerful optimization framework for optimizing the algorithm parameters or any subset of them to optimally balance among the time cost, energy cost, and convergence error.
Finally, numerical results reveal several important design insights for FL.
Firstly, the derived convergence error of GenQSGD effectively characterizes its training loss and test accuracy and hence plays a crucial role in the optimal design of FL.
Secondly, a trade-off among the time cost, energy cost, and convergence error can be achieved by properly choosing the algorithm parameters of FL.
Last but not least, the more algorithm parameters can be optimized, the better trade-off can be archived.
Therefore, the proposed optimization framework for FL design lays the foundation for optimizing the overall performance of federated edge learning.

\section*{Appendix A: Proof of Theorem~\ref{Thm:Convergence}}\label{Prf:Convergence_Q_n}
To prove Theorem~\ref{Thm:Convergence}, we first show the following three lemmas.
For ease of exposition, we define $K_{\max}\triangleq\max_{n\in\mathcal N}K_n$ in the following.
\begin{Lem}\label{Lem:Smoothness}
Suppose that Assumption~\ref{Asump:Quantization} and Assumption~\ref{Asump:Smoothness} are satisfied and the step size $\gamma^{(k_0)}\in\left(0,\frac{1}{L}\right]$ for all $k_0\in\mathcal K_0$.
Then, for all $\mathbf K\in\mathbb Z_+^{N+1}$ and $B\in\mathbb Z_+$,
$\left\{\hat{\mathbf x}^{(k_0)}:k_0\in\mathcal K_0\right\}$ and $\left\{\bar{\mathbf x}^{(k_0, K_{\max})}:\right.\\\left.k_0\in\mathcal K_0\right\}$ generated by GenQSGD satisfy:
\begin{small}\begin{align}
&\mathbb E\left[f\left(\hat{\mathbf x}^{(k_0+1)}\right)\right]
\leq\mathbb E\left[f\left(\bar{\mathbf x}^{(k_0, K_{\max})}\right)\right]
+\frac{L}{2}\mathbb E\left[\left\|\hat{\mathbf x}^{(k_0+1)}-\bar{\mathbf x}^{(k_0, K_{\max})}\right\|_2^2\right].\nonumber
\end{align}\end{small}
\end{Lem}
\begin{IEEEproof}
For all $\mathbf x,\mathbf y\in\mathbb R^D$, we have:
\begin{small}\begin{align}\label{eq:Lem_Smoothness_1}
\left\|\nabla{f({\mathbf x})}-\nabla{f({\mathbf y})}\right\|_2
{\overset{(a)}{=}}&\left\|{\frac{1}{N}\sum_{n\in\mathcal N}\left(\nabla f_n({\mathbf x})-\nabla f_n({\mathbf y})\right)}\right\|_2
{\overset{(b)}{\leq}}\frac{1}{N}\sum_{n\in\mathcal N}\left\|{\nabla f_n({\mathbf x})-\nabla f_n({\mathbf y})}\right\|_2
{\overset{(c)}{\leq}}L\left\|\mathbf x-\mathbf y\right\|_2,
\end{align}\end{small}where (a) follows from \eqref{eq:expected_loss},
(b) follows from the triangle inequality,
and (c) follows from Assumption~\ref{Asump:Smoothness}.
Note that \eqref{eq:Lem_Smoothness_1} indicates that
$\nabla f(\cdot)$ is Lipschitz continuous.
Thus, we have:
\begin{small}\begin{align}\label{eq:Lem_Smoothness_2}
f\!\left(\!\hat{\mathbf x}^{(k_0\!+\!1)}\!\right)
\!\leq\!f\!\left(\!\bar{\mathbf x}^{(k_0, K_{\max})}\!\right)
\!+\!\nabla f\!\left(\!\bar{\mathbf x}^{(k_0, K_{\max})}\!\right)^T
\!\!\left(\!\hat{\mathbf x}^{(k_0\!+\!1)}-\bar{\mathbf x}^{(k_0, K_{\max})}\!\right)
\!+\!\frac{L}{2}\!\left\|\hat{\mathbf x}^{(k_0\!+\!1)}-\bar{\mathbf x}^{(k_0, K_{\max})}\right\|_2^2\!,k_0\!\in\!\mathcal K_0.
\end{align}\end{small}Besides, we have:
\begin{small}\begin{align}\label{eq:Lem_Smoothness_3}
\mathbb E\left[\hat{\mathbf x}^{(k_0+1)}\right]
{\overset{(d)}{=}}&\hat{\mathbf x}^{(k_0)}+\mathbb E\left[\mathbf Q(\Delta\hat{\mathbf x}^{(k_0)};s_0)\right]
{\overset{(e)}{=}}\hat{\mathbf x}^{(k_0)}+\mathbb E\left[\frac{1}{N}\sum_{n\in\mathcal N}\mathbf Q\left({\bf x}_n^{(k_0,K_n)}-\hat{\mathbf x}^{(k_0)};s_n\right)\right]\nonumber\\[-5pt]
{\overset{(f)}{=}}&\hat{\mathbf x}^{(k_0)}+\frac{1}{N}\sum_{n\in\mathcal N}\left({\bf x}_n^{(k_0,K_n)}-\hat{\mathbf x}^{(k_0)}\right)
{\overset{(g)}{=}}\bar{\mathbf x}^{(k_0,K_{\max})},\ k_0\in\mathcal K_0,
\end{align}\end{small}where (d) is due to \eqref{eq:RecoveredLocalModel},
(e) is due to \eqref{eq:GlobalUpdate} and Assumption~\ref{Asump:Quantization} (i),
(f) is due to Assumption~\ref{Asump:Quantization} (i),
and (g) is due to \eqref{eq:local_update_redefine} and \eqref{eq:x_avg}.
By taking expectation of both sides of \eqref{eq:Lem_Smoothness_2}, we have:
\begin{small}\begin{align}
&\mathbb E\left[f\left(\hat{\mathbf x}^{(k_0+1)}\right)\right]
\leq\mathbb E\left[f\left(\bar{\mathbf x}^{(k_0, K_{\max})}\right)\right]
+\mathbb E\left[\nabla f\left(\bar{\mathbf x}^{(k_0, K_{\max})}\right)^T\left(\hat{\mathbf x}^{(k_0+1)}-\bar{\mathbf x}^{(k_0, K_{\max})}\right)\right]\nonumber\\[-5pt]
&{\quad}+\frac{L}{2}\mathbb E\left[\left\|\hat{\mathbf x}^{(k_0+1)}-\bar{\mathbf x}^{(k_0, K_{\max})}\right\|_2^2\right]
{\overset{(h)}{=}}\mathbb E\left[f\left(\bar{\mathbf x}^{(k_0, K_{\max})}\right)\right]
+\frac{L}{2}\mathbb E\left[\left\|\hat{\mathbf x}^{(k_0+1)}-\bar{\mathbf x}^{(k_0, K_{\max})}\right\|_2^2\right],\ k_0\in\mathcal K_0,\nonumber
\end{align}\end{small}where (h) is due to \eqref{eq:Lem_Smoothness_3}.
Therefore, we can show Lemma~\ref{Lem:Smoothness}.\end{IEEEproof}
\begin{Lem}\label{Lem:TermBound1}
Suppose that Assumptions~\ref{Asump:IID}, \ref{Asump:Smoothness}, \ref{Asump:BoundedVariances}, and \ref{Asump:BoundedSecondMoments} are satisfied and the step size $\gamma^{(k_0)}\in\left(0,\frac{1}{L}\right]$ for all $k_0\in\mathcal K_0$.
Then, for all $\mathbf K\in\mathbb Z_+^{N+1}$ and $B\in\mathbb Z_+$,
$\left\{\hat{\mathbf x}^{(k_0)}:k_0\in\mathcal K_0\right\}$ and
$\left\{\bar{\mathbf x}^{(k_0,k)}:k_0\in\mathcal K_0,\right.\\\left.k\in\mathcal K_{\max}\right\}$
generated by GenQSGD satisfy:
\begin{small}\begin{align}
\mathbb E\left[f\left(\bar{\mathbf x}^{(k_0, K_{\max})}\right)\right]
\leq&\mathbb E\left[f\left(\hat{\mathbf x}^{(k_0)}\right)\right]-\frac{\gamma^{(k_0)}}{2}\sum_{k\in\mathcal K_{\max}}\frac{N_k}{N}\mathbb E\left[\left\|\nabla f\left(\bar{\mathbf x}^{(k_0, k-1)}\right)\right\|_2^2\right]\nonumber\\[-5pt]
&+\left(2G^2L^2K_{\max}^2\left(\gamma^{(k_0)}\right)^3+\frac{L\sigma^2\left(\gamma^{(k_0)}\right)^2}{2NB}\right)\sum_{k\in\mathcal K_{\max}}\frac{N_k}{N}.\nonumber
\end{align}\end{small}
\end{Lem}
\begin{IEEEproof}
Assumptions~\ref{Asump:IID}, \ref{Asump:Smoothness}, \ref{Asump:BoundedVariances}, and \ref{Asump:BoundedSecondMoments} in this paper indicate that Assumption 1 and Assumption 2 in \cite{YuHao} are satisfied.
By Assumption~\ref{Asump:BoundedVariances}, we have
$\mathbb E\left[\left\|\frac{1}{B}\sum_{\xi_i\in\mathcal B_n^{(k_0,k_n)}}\nabla{F\left({\mathbf x};\xi_i\right)}-\nabla f_n\left({\mathbf x}\right)\right\|_2^2\right]
\leq\frac{\sigma^2}{B},\ n\in\mathcal N,\ k_0\in\mathcal K_0,\ k_n\in\mathcal K_n$
for all $\mathbf x\in\mathbb R^D$ \cite[(5.8)]{LargeScaleML}.
Therefore, by slightly extending the proof for (25) in \cite{YuHao}, we can show:
\begin{small}\begin{align}\label{eq:lem_termbound1_2}
\mathbb E\left[f\left(\bar{\mathbf x}^{(k_0,k)}\right)\right]
\leq&\mathbb \mathbb E\left[f\left(\bar{\mathbf x}^{(k_0,k-1)}\right)\right]-\frac{\gamma^{(k_0)}N_k}{2N}\mathbb E\left[\left\|\nabla f\left(\bar{\mathbf x}^{(k_0, k-1)}\right)\right\|_2^2\right]\nonumber\\[-5pt]
&+\frac{2N_kG^2L^2K_{\max}^2\left(\gamma^{(k_0)}\right)^3}{N}
+\frac{N_kL\sigma^2\left(\gamma^{(k_0)}\right)^2}{2N^2B},\ k_0\in\mathcal K_0,\ k\in\mathcal K_{\max},
\end{align}\end{small}which relies on Assumptions 1 and 2 in \cite{YuHao}, and the constant step size condition.
By noting that $\bar{\mathbf x}^{(k_0, 0)}=\hat{\mathbf x}^{(k_0)}, k_0\in\mathcal K_0$ and summing \eqref{eq:lem_termbound1_2} over $k\in\mathcal K_{\max}$, we can show Lemma~\ref{Lem:TermBound1}.\end{IEEEproof}
\begin{Lem}\label{Lem:TermBound2}
Suppose Assumption~\ref{Asump:Quantization} and Assumption~\ref{Asump:BoundedSecondMoments} are satisfied and the step size $\gamma^{(k_0)}\in\left(0,\frac{1}{L}\right]$ for all $k_0\in\mathcal K_0$.
Then, for all $\mathbf K\in\mathbb Z_+^{N+1}$ and $B\in\mathbb Z_+$,
$\left\{\hat{\mathbf x}^{(k_0)}:k_0\in\mathcal K_0\right\}$
and $\left\{\bar{\mathbf x}^{(k_0,k)}:k_0\in\mathcal K_0,\right.\\\left.k\in\mathcal K_{\max}\right\}$ generated by GenQSGD satisfy:
\begin{small}\begin{align}\label{eq:lem_termbound2}
&\mathbb E\left[\left\|\hat{\mathbf x}^{(k_0+1)}-\bar{\mathbf x}^{(k_0, K_{\max})}\right\|_2^2\right]
\leq\frac{2G^2\left(\gamma^{(k_0)}\right)^2}{N}\sum_{n\in\mathcal N}(q_{s_n}+q_{s_0}+q_{s_0}q_{s_n})K_n^2.
\end{align}\end{small}
\end{Lem}
\begin{IEEEproof}
First, we have:
\begin{small}\begin{align}\label{eq:lem_termbound2_1}
\mathbb E\!\left[\left\|\mathbf Q\left(\mathbf x_n^{(k_0,K_n)}\!-\!\hat{\mathbf x}^{(k_0)};s_n\right)\right\|_2^2\right]
{\overset{(a)}{\leq}}&\mathbb E\!\left[\left\|\mathbf Q\left(\mathbf x_n^{(k_0,K_n)}\!-\!\hat{\mathbf x}^{(k_0)};s_n\right)\!-\!\left(\mathbf x_n^{(k_0,K_n)}\!-\!\hat{\mathbf x}^{(k_0)}\right)\right\|_2^2\right]
\!+\!\mathbb E\!\left[\left\|\mathbf x_n^{(k_0,K_n)}\!-\!\hat{\mathbf x}^{(k_0)}\right\|_2^2\right]\nonumber\\[-5pt]
{\overset{(b)}{\leq}}&\left(q_{s_n}+1\right)\mathbb E\left[\left\|\mathbf x_n^{(k_0,K_n)}-\hat{\mathbf x}^{(k_0)}\right\|_2^2\right],\ n\in\mathcal N,\ k_0\in\mathcal K_0,
\end{align}\end{small}where (a) is deduced from the definition of variance, and (b) follows by Assumption~\ref{Asump:Quantization} (ii).
Besides, we have:
\begin{small}\begin{align}\label{eq:lem_termbound2_2}
&\mathbb E\left[\left\|\hat{\mathbf x}^{(k_0+1)}-\bar{\mathbf x}^{(k_0, K_{\max})}\right\|_2^2\right]
{\overset{(c)}{=}}\mathbb E\left[\left\|\hat{\mathbf x}^{(k_0)}+\mathbf Q\left(\frac{1}{N}\sum_{n\in\mathcal N}\mathbf Q\left({\bf x}_n^{(k_0,K_n)}-\hat{\mathbf x}^{(k_0)};s_n\right);s_0\right)-\frac{1}{N}\sum_{n\in\mathcal N}{\bf x}_n^{(k_0,K_n)}\right\|_2^2\right]\nonumber\\[-5pt]
=&\mathbb E\left[\left\|\mathbf Q\left(\frac{1}{N}\sum_{n\in\mathcal N}\mathbf Q\left(\mathbf x_n^{(k_0,K_n)}-\hat{\mathbf x}^{(k_0)};s_n\right);s_0\right)-\frac{1}{N}\sum_{n\in\mathcal N}\mathbf Q\left(\mathbf x_n^{(k_0,K_n)}-\hat{\mathbf x}^{(k_0)};s_n\right)\right.\right.\nonumber\\[-5pt]
&\left.\left.+\frac{1}{N}\sum_{n\in\mathcal N}\left(\mathbf Q\left(\mathbf x_n^{(k_0,K_n)}-\hat{\mathbf x}^{(k_0)};s_n\right)-\left(\mathbf x_n^{(k_0,K_n)}-\hat{\mathbf x}^{(k_0)}\right)\right)\right\|_2^2\right]\nonumber\\[-5pt]
{\overset{(d)}{\leq}}&2\mathbb E\left[\left\|\mathbf Q\left(\frac{1}{N}\sum_{n\in\mathcal N}\mathbf Q\left(\mathbf x_n^{(k_0,K_n)}-\hat{\mathbf x}^{(k_0)};s_n\right);s_0\right)-\frac{1}{N}\sum_{n\in\mathcal N}\mathbf Q\left(\mathbf x_n^{(k_0,K_n)}-\hat{\mathbf x}^{(k_0)};s_n\right)\right\|_2^2\right]\nonumber\\[-5pt]
&+2\mathbb E\left[\left\|\frac{1}{N}\sum_{n\in\mathcal N}\left(\mathbf Q\left(\mathbf x_n^{(k_0,K_n)}-\hat{\mathbf x}^{(k_0)};s_n\right)-\left(\mathbf x_n^{(k_0,K_n)}-\hat{\mathbf x}^{(k_0)}\right)\right)\right\|_2^2\right]\nonumber\\[-4pt]
{\overset{(e)}{\leq}}&2q_{s_0}\mathbb E\!\left[\left\|\frac{1}{N}\!\sum_{n\in\mathcal N}\!\mathbf Q\left(\mathbf x_n^{(k_0,K_n)}\!-\!\hat{\mathbf x}^{(k_0)};s_n\right)\right\|_2^2\right]
\!+\!2\mathbb E\!\left[\left\|\frac{1}{N}\!\sum_{n\in\mathcal N}\!\left(\mathbf Q\left(\mathbf x_n^{(k_0,K_n)}\!-\!\hat{\mathbf x}^{(k_0)};s_n\right)\!-\!\left(\mathbf x_n^{(k_0,K_n)}\!-\!\hat{\mathbf x}^{(k_0)}\right)\right)\right\|_2^2\right]\nonumber\\[-5pt]
{\overset{(f)}{\leq}}&\frac{2q_{s_0}}{N}\sum_{n\in\mathcal N}\mathbb E\!\left[\left\|\mathbf Q\left(\mathbf x_n^{(k_0,K_n)}\!-\!\hat{\mathbf x}^{(k_0)};s_n\right)\right\|_2^2\right]
+\frac{2}{N}\sum_{n\in\mathcal N}\mathbb E\!\left[\left\|\mathbf Q\left(\mathbf x_n^{(k_0,K_n)}\!-\!\hat{\mathbf x}^{(k_0)};s_n\right)\!-\!\left(\mathbf x_n^{(k_0,K_n)}\!-\!\hat{\mathbf x}^{(k_0)}\right)\right\|_2^2\right]\nonumber\\[-5pt]
{\overset{(g)}{\leq}}&\frac{2q_{s_0}}{N}\sum_{n\in\mathcal N}(q_{s_n}\!+\!1)\mathbb E\left[\left\|\mathbf x_n^{(k_0,K_n)}-\hat{\mathbf x}^{(k_0)}\right\|_2^2\right]
\!+\!\frac{2}{N}\sum_{n\in\mathcal N}q_{s_n}\mathbb E\left[\left\|\mathbf x_n^{(k_0,K_n)}-\hat{\mathbf x}^{(k_0)}\right\|_2^2\right],\ k_0\in\mathcal K_0,
\end{align}\end{small}where (c) follows by \eqref{eq:RecoveredLocalModel}, \eqref{eq:GlobalUpdate}, \eqref{eq:local_update_redefine}, and \eqref{eq:x_avg},
(d) and (f) follow by the inequality $\|\sum_{i=1}^n{\bf z}_i\|^2\leq n\sum_{i=1}^n{\|{\bf z}_i\|^2}$,
(e) follows by Assumption~\ref{Asump:Quantization} (ii),
and (g) is due to Assumption~\ref{Asump:Quantization} (ii) and \eqref{eq:lem_termbound2_1}.
By \eqref{eq:local_update} and ${\mathbf x}_n^{(k_0,0)}=\hat{\mathbf x}^{(k_0)}$, we have:
\begin{small}\begin{align}\label{eq:lem_termbound2_3}
&\mathbb E\left[\left\|\mathbf x_n^{(k_0,K_n)}-\hat{\mathbf x}^{(k_0)}\right\|_2^2\right]
\leq\left(\gamma^{(k_0)}\right)^2\mathbb E\left[\left\|\sum_{k\in\mathcal K_n}\nabla\hat F\left({\bf x}_n^{(k_0,k-1)}\right)\right\|_2^2\right]
{\overset{(h)}{\leq}}\left(\gamma^{(k_0)}\right)^2K_n^2G^2,\ k_0\in\mathcal K_0,
\end{align}\end{small}where (h) follows by Assumption~\ref{Asump:BoundedSecondMoments} and the inequality $\|\sum_{i=1}^n{\bf z}_i\|^2\leq n\sum_{i=1}^n{\|{\bf z}_i\|^2}$.
By Assumption~\ref{Asump:Quantization} (ii), we can substitute \eqref{eq:lem_termbound2_3} into \eqref{eq:lem_termbound2_2} and show Lemma~\ref{Lem:TermBound2}.
\end{IEEEproof}

By Lemma~\ref{Lem:Smoothness}, Lemma~\ref{Lem:TermBound1}, and Lemma~\ref{Lem:TermBound2}, we have:
\begin{small}\begin{align}\label{eq:Thm_Final_1}
\mathbb E\left[f\left(\hat{\mathbf x}^{(k_0+1)}\right)\right]
\leq&\mathbb E\left[f\left(\hat{\mathbf x}^{(k_0)}\right)\right]-\frac{\gamma^{(k_0)}}{2}\sum_{k\in\mathcal K_{\max}}\frac{N_k}{N}\mathbb E\left[\left\|\nabla f\left(\bar{\mathbf x}^{(k_0, k-1)}\right)\right\|_2^2\right]
+\Bigg(2G^2L^2K_{\max}^2\left(\gamma^{(k_0)}\right)^3\nonumber\\[-5pt]
&+\frac{L\sigma^2\left(\gamma^{(k_0)}\right)^2}{2NB}\Bigg)\sum_{k\in\mathcal K_{\max}}\frac{N_k}{N}+\frac{LG^2\left(\gamma^{(k_0)}\right)^2\sum_{n\in\mathcal N}q_{s_0,s_n}K_n^2}{N},\ k_0\in\mathcal K_0.
\end{align}\end{small}Then, by \eqref{eq:Thm_Final_1} and \eqref{eq:N_k}, we have:
\begin{small}\begin{align}\label{eq:Thm_Final_2}
&\frac{\gamma^{(k_0)}}{2}\sum_{k\in\mathcal K_{\max}}\frac{N_k}{N}\mathbb E\left[\left\|\nabla f\left(\bar{\mathbf x}^{(k_0, k-1)}\right)\right\|_2^2\right]
\leq\mathbb E\left[f\left(\hat{\mathbf x}^{(k_0)}\right)\right]-\mathbb E\left[f\left(\hat{\mathbf x}^{(k_0+1)}\right)\right]\nonumber\\[-5pt]
&+\left(2G^2L^2K_{\max}^2\gamma^{(k_0)}\sum_{n\in\mathcal N}{K_n}
+\frac{L\sigma^2\sum_{n\in\mathcal N}{K_n}}{2NB}+LG^2\sum_{n\in\mathcal N}q_{s_0,s_n}K_n^2\right)\frac{\left(\gamma^{(k_0)}\right)^2}{N},\ k_0\in\mathcal K_0.
\end{align}\end{small}Summing both sides of \eqref{eq:Thm_Final_2} over $k_0\in\mathcal K_0$, we readily show Theorem~\ref{Thm:Convergence}.

\section*{Appendix B: Proof of Lemma~\ref{Lem:Convergence_exp}}\label{Prf:Convergence_exp}
We readily show \eqref{eq:Convergence_exp} by substituting $\sum_{k_0\in\mathcal K_0}\gamma^{(k_0)}{\overset{(a)}{=}}\frac{\gamma_E\left(1-\rho_E^{K_0}\right)}{1-\rho_E}$,
$\sum_{k_0\in\mathcal K_0}\left(\gamma^{(k_0)}\right)^2{\overset{(b)}{=}}\frac{\gamma_E^2\left(1-\rho_E^{2K_0}\right)}{1-\rho_E^2}$,
and $\sum_{k_0\in\mathcal K_0}\left(\gamma^{(k_0)}\right)^3{\overset{(c)}{=}}\frac{\gamma_E^3\left(1-\rho_E^{3K_0}\right)}{1-\rho_E^3}$ into \eqref{eq:Convergence_RHS}, where (a), (b), and (c) follow by summing \eqref{eq:exp_rule} over $k_0\in\mathcal K_0$.
Then, the limit of $C_{E}(\mathbf K,B,\mathbf\Gamma)$ can be easily derived.
Note that for any $K_0\in\mathbb Z_+$ and constant $A_E\in(1,+\infty)$,
there exists a $\rho_E\in(0,1)$ such that $\frac{1}{\sum_{k_0\in\mathcal K_0}\rho_E^{k_0}}\leq\frac{A_E}{K_0}$.
By substituting \eqref{eq:exp_rule} into \eqref{eq:Convergence_RHS} and by $\frac{1}{\sum_{k\in\mathcal K_0}\rho_E^{k_0}}\leq\frac{A_E}{K_0}$,
$\frac{\sum_{k_0\in\mathcal K_0}\rho_E^{3k_0}}{\sum_{k_0\in\mathcal K_0}\rho_E^{k_0}}<1$,
and $\frac{\sum_{k_0\in\mathcal K_0}\rho_E^{2k_0}}{\sum_{k_0\in\mathcal K_0}\rho_E^{k_0}}<1$, we have:
\begin{small}\begin{align}\label{eq:Lem2_prf1}
C_A(\mathbf K,B,\mathbf\Gamma)
&<\frac{c_1A_E}{\gamma_EK_0\sum_{n\in\mathcal N}K_n}
+c_2\gamma_E^2\max_{n\in\mathcal N}K_n^2
+\frac{c_3\gamma_E}{B}+\frac{c_4\gamma_E\sum_{n\in\mathcal N}q_{s_0,s_n}K_n^2}{\sum_{n\in\mathcal N}K_n}.
\end{align}\end{small}By substituting $\mathbf\Gamma$ given by \eqref{eq:exp_rule} with $\gamma_C=\frac{\sqrt{N}}{L\sqrt{K_0\bar K}}$ and $K_n\!=\!\bar K$, $q_{s_0,s_n}\!=\!\frac{1}{N\bar K}$, $n\!\in\!\mathcal N$ with $\bar K\leq\frac{\left(K_0\bar K\right)^{1/4}}{N^{3/4}}$ into \eqref{eq:Lem2_prf1},
we have $C_A(\mathbf K,B,\mathbf\Gamma)=\mathcal O\left(K_0^{-\frac{1}{2}}\right)$.
Therefore, we can show Lemma~\ref{Lem:Convergence_exp}.

\section*{Appendix C: Proof of Lemma~\ref{Lem:Convergence_dim}}\label{Prf:Convergence_dim}
By summing \eqref{eq:dim_rule_sp} over $k_0\in\mathcal K_0$,
we have $\sum_{k_0\in\mathcal K_0}\!\!\gamma^{(k_0)}
{\overset{(a)}{>}}\rho_D\gamma_D\int_{1}^{K_0\!+\!1}\!\!\!\!\frac{1}{x+\rho_D}dx
\!=\!\rho_D\gamma_D\ln\left(\frac{K_0\!+\!\rho_D\!+\!1}{\rho_D+1}\right)$,
$\sum_{k_0\in\mathcal K_0}\left(\gamma^{(k_0)}\right)^2
{\overset{(b)}{<}}\left(\frac{\rho_D\gamma_D}{\rho_D+1}\right)^2
+\rho_D^2\gamma_D^2\int_{1}^{K_0}\frac{1}{(x+\rho_D)^2}dx
=\frac{\rho_D^2\gamma_D^2}{(\rho_D+1)^2}+\frac{\rho_D^2\gamma_D^2}{\rho_D+1}
-\frac{\rho_D^2\gamma_D^2}{K_0+\rho_D}$,
and $\sum_{k_0\in\mathcal K_0}\left(\gamma^{(k_0)}\right)^3
\\{\overset{(c)}{<}}\left(\frac{\rho_D\gamma_D}{\rho_D+1}\right)^3
+\rho_D^3\gamma_D^3\int_{1}^{K_0}\frac{1}{(x+\rho_D)^3}dx
=\frac{\rho_D^3\gamma_D^3}{(\rho_D+1)^3}+\frac{\rho_D^3\gamma_D^3}{2(\rho_D+1)^2}
-\frac{\rho_D^3\gamma_D^3}{2(K_0+\rho_D)^2}$,
where (a) is due to $\frac{1}{k_0+\rho_D}>\frac{1}{x+\rho_D},\ x\in\left(k_0,k_0+1\right]$ for all $k_0\in\mathcal K_0$, (b) is due to $\frac{1}{\left(k_0+\rho_D\right)^2}<\frac{1}{\left(x+\rho_D\right)^2},\ x\in\left[k_0-1,k_0\right)$ for all $k_0\in\mathcal K_0\backslash\{1\}$,
and (c) is due to $\frac{1}{\left(k_0+\rho_D\right)^3}<\frac{1}{\left(x+\rho_D\right)^3},\ x\in\left[k_0-1,k_0\right)$ for all $k_0\in\mathcal K_0\backslash\{1\}$, respectively.
Thus, we have
$C_A(\mathbf K,B,\mathbf\Gamma)
<\frac{b_1c_1}{\ln\left(\frac{K_0+\rho_D+1}{\rho_D+1}\right)\sum_{n\in\mathcal N}K_n}
+\frac{c_2\left(b_2-\frac{\rho_D^2\gamma_D^2}{2\left(K_0+\rho_D\right)^2}\right)K_{\max}^2}{\ln\left(\frac{K_0+\rho_D+1}{\rho_D+1}\right)}
\\+\frac{b_3-\frac{\rho_D\gamma_D}{K_0+\rho_D}}{\ln\left(\frac{K_0+\rho_D+1}{\rho_D+1}\right)}
\left(\frac{c_3}{B}+\frac{c_4\sum_{n\in\mathcal N}q_{s_0,s_n}K_n^2}{\sum_{n\in\mathcal N}K_n}\right)
<C_D(\mathbf K,B,\mathbf\Gamma)$.
Then, the limit for $\mathbf\Gamma$ satisfying \eqref{eq:dim_rule} can be easily derived.
Note that for any $K_0\in\mathbb Z_+$ and constant $A_D\in(1,+\infty)$,
there exists a $\rho_D\in\mathbb R_+$ such that $\frac{1}{\sum_{k_0\in\mathcal K_0}\frac{\rho_D}{k_0+\rho_D}}\leq\frac{A_D}{K_0}$.
By substituting \eqref{eq:dim_rule_sp} into \eqref{eq:Convergence_RHS} and $\frac{1}{\sum_{k_0\in\mathcal K_0}\frac{\rho_D}{k_0+\rho_D}}\leq\frac{A_D}{K_0}$,
$\frac{\sum_{k_0\in\mathcal K_0}\left(\frac{\rho_D}{k_0+\rho_D}\right)^3}{\sum_{k_0\in\mathcal K_0}\frac{\rho_D}{k_0+\rho_D}}<1$,
and $\frac{\sum_{k_0\in\mathcal K_0}\left(\frac{\rho_D}{k_0+\rho_D}\right)^2}{\sum_{k_0\in\mathcal K_0}\frac{\rho_D}{k_0+\rho_D}}<1$, we have:
\begin{small}\begin{align}\label{eq:Lem3_prf1}
C_A(\mathbf K,B,\mathbf\Gamma)
&<\frac{c_1A_D}{\gamma_DK_0\sum_{n\in\mathcal N}K_n}
+c_2\gamma_D^2\max_{n\in\mathcal N}K_n^2
+\frac{c_3\gamma_D}{B}
+\frac{c_4\gamma_D\sum_{n\in\mathcal N}q_{s_0,s_n}K_n^2}{\sum_{n\in\mathcal N}K_n},
\end{align}\end{small}By substituting $\mathbf\Gamma$ given by \eqref{eq:dim_rule_sp} with $\gamma_C\!=\!\frac{\sqrt{N}}{L\sqrt{K_0\bar K}}$ and $K_n\!=\!\bar K$, $q_{s_0,s_n}=\frac{1}{N\bar K}$, $n\!\in\!\mathcal N$ with $\bar K\leq\frac{\left(K_0\bar K\right)^{1/4}}{N^{3/4}}$ into \eqref{eq:Lem3_prf1},
we have $C_A(\mathbf K,B,\mathbf\Gamma)=\mathcal O\left(K_0^{-\frac{1}{2}}\right)$.
Therefore, we can show Lemma~\ref{Lem:Convergence_dim}.


\section*{Appendix D: Proof of Theorem~\ref{Thm:Equivalent_Prob_exp}}\label{Prf:Equivalent_Prob_exp}
First, we construct the following problem.
\begin{Prob}[Equivalent Problem of Problem~\ref{Prob:fixed_step_size} with $m=E$]\label{Prob:exp_eq_tmp_tmp}
For any given $\gamma_E\in\left(0,\frac{1}{L}\right]$ and $\rho_E\in(0,1)$,
\begin{small}\begin{align}
&\min_{\substack{\mathbf K\succ\mathbf0,B,T_1,T_2>0}}{\quad}E(\mathbf K,B)\nonumber\\[-5pt]
&\mathrm{s.t.}{\quad}\eqref{eq:fix_epi_cons1},\ \eqref{eq:fix_epi_cons2},\ \eqref{eq:fix_epi_cons_T},\nonumber\\[-5pt]
&\frac{a_1c_1}{\left(\!1\!-\!\rho_E^{K_0}\!\right)\sum_{n\in\mathcal N}K_n}
\!+\!\frac{a_2c_2\left(\!1\!-\!\rho_E^{3K_0}\!\right)\max_{n\in\mathcal N}K_n^2}
{\left(\!1\!-\!\rho_E^{K_0}\!\right)}
\!+\!\frac{a_3\left(\!1\!-\!\rho_E^{2K_0}\!\right)}{\left(\!1\!-\!\rho_E^{K_0}\!\right)}
\!\left(\!\frac{c_3}{B}
\!+\!\frac{c_4\sum_{n\in\mathcal N}q_{s_0,s_n}K_n^2}{\sum_{n\in\mathcal N}K_n}\!\right)\!
\!\leq\!C_{\max}.\label{eq:exp_cons_C_tmp}
\end{align}\end{small}
\end{Prob}
Following the proof of Theorem~\ref{Thm:Equivalent_Prob_cons}, by contradiction,
we can show that $\left(\mathbf K^*,B^*,T_1^*,T_2^*\right)$ is optimal for Probelm~\ref{Prob:exp_eq_tmp_tmp}, where $\left(\mathbf K^*,B^*,T_1^*,T_2^*\right)$ is from an optimal point of
Problem~\ref{Prob:exp_eq}.

Then, it remains to show the equivalence between Problem~\ref{Prob:exp_eq_tmp_tmp} and Problem~\ref{Prob:fixed_step_size} with $m=E$.
We introduce an optimization variable $X_0\in\mathbb R_+$, add the constraint $X_0=\rho_E^{K_0}$, and replace $X_0=\rho_E^{K_0}$ with the constraints in \eqref{eq:exp_cons_X0_1}, \eqref{eq:exp_cons_X0_2}, and \eqref{eq:cons_X_0}.
Consequently, Problem~\ref{Prob:exp_eq_tmp_tmp} is equivalent to Problem~\ref{Prob:fixed_step_size} with $m=E$ and hence $\left(\mathbf K^*, B^*\right)$ is an optimal point of Problem~\ref{Prob:fixed_step_size} with $m=E$.
Therefore, we can show Theorem~\ref{Thm:Equivalent_Prob_exp}.

\section*{Appendix E: Proof of Theorem~\ref{Thm:Equivalent_Prob_adj}}\label{Prf:Equivalent_Prob_adj}
First, we construct the following problem.
\begin{Prob}[Equivalent Problem of Problem~\ref{Prob:adj_step_size}]\label{Prob:adj_step_size_tmp}
\begin{small}\begin{align}
\min_{\substack{\mathbf K\succ\mathbf 0,B,\gamma>0}}&{\quad}E(\mathbf K,B)\nonumber\\[-5pt]
\mathrm{s.t.}&{\quad}\eqref{eq:Cons_time},\ \eqref{eq:adj_step_size_epi_cons_gamma},\nonumber\\
&\frac{c_1}{\gamma K_0\sum_{n\in\mathcal N}K_n}
+c_2\gamma^2\max_{n\in\mathcal N}K_n^2
+\frac{c_3\gamma}{B}
+\frac{c_4\gamma\sum_{n\in\mathcal N}q_{s_0,s_n}K_n^2}{\sum_{n\in\mathcal N}K_n}\leq C_{\max}.
\end{align}\end{small}
\end{Prob}
Following the proof of Theorem~\ref{Thm:Equivalent_Prob_cons}, by contradiction,
we can show that $\left(\mathbf K^*,B^*,\gamma^*\right)$ is optimal for Problem~\ref{Prob:adj_step_size_tmp}, where $\left(\mathbf K^*,B^*,\gamma^*\right)$ is from an optimal point of Problem~\ref{Prob:adj_step_size}.

Then, it remains to show the equivalence between Problem~\ref{Prob:adj_step_size_tmp} and Problem~\ref{Prob:adj_step_size}.
Suppose that there exists $\left(\mathbf K^\dag,B^\dag,\mathbf\Gamma^\dag\right)\neq\left(\mathbf K^*,B^*,\gamma^*\mathbf1\right)$ satisfying all constraints of Problem~\ref{Prob:adj_step_size} and $E\left(\mathbf K^\dag,B^\dag\right)<E\left(\mathbf K^*,B^*\right)$.
Construct $\hat\gamma=\frac{\mathbf1^{\text{T}}\mathbf\Gamma^\dag}{K_0^\dag}$.
As $\mathbf\Gamma^\dag$ satisfies \eqref{eq:Cons_gamma}, $\hat\gamma$ satisfies the constraint in \eqref{eq:adj_step_size_epi_cons_gamma}.
By Lemma~\ref{Lem:optimal_step_size}, $C_A(\mathbf K^\dag,B^\dag,\hat\gamma\mathbf1)\leq C_A(\mathbf K^\dag,B^\dag,\mathbf\Gamma^\dag)\leq C_{\max}$,
implying that $\left(\mathbf K^\dag,B^\dag,\hat\gamma\right)$ satisfies all constraints of Problem~\ref{Prob:adj_step_size_tmp} and $E\left(\mathbf K^\dag,B^\dag\right)<E\left(\mathbf K^*,B^*\right)$,
which in turn contradicts with the optimality of $\left(\mathbf K^*,B^*,\gamma^*\right)$ for Problem~\ref{Prob:adj_step_size_tmp}.
Thus, by contradiction, we can show that $\left(\mathbf K^*,B^*,\gamma^*\mathbf1\right)$ is optimal for Problem~\ref{Prob:adj_step_size}.
Therefore, we can show Theorem~\ref{Thm:Equivalent_Prob_adj}.

\normalem
\bibliographystyle{IEEEtran}      
\bibliography{IEEEabrv,GenQSGD}                        

\end{document}